\documentclass[letterpaper]{article} 
\usepackage{aaai23}  
\usepackage{times}  
\usepackage{helvet}  
\usepackage{courier}  
\usepackage[hyphens]{url}  
\usepackage{graphicx} 
\usepackage{natbib}  
\usepackage{caption} 
\usepackage{Symbol}
\usepackage{caption}
\usepackage{subfigure}
\usepackage{amsthm}
\usepackage{amsmath}
\usepackage{amssymb}
\usepackage{bm}
\usepackage{color, colortbl}
\definecolor{LightCyan}{rgb}{0.88,1,1}

\newtheorem{assumption}{Assumption}
\newtheorem{lemma}{Lemma}
\newtheorem{theorem}{Theorem}
\newtheorem{remark}{Remark}

\usepackage{algorithm}
\usepackage{algorithmic}

\usepackage{newfloat}
\usepackage{listings}
\DeclareCaptionStyle{ruled}{labelfont=normalfont,labelsep=colon,strut=off} 
\floatstyle{ruled}
\newfloat{listing}{tb}{lst}{}
\floatname{listing}{Listing}
\title{Achieving Zero Constraint Violation for Constrained Reinforcement Learning via Conservative Natural Policy Gradient Primal-Dual Algorithm}
\author{
    Qinbo Bai\textsuperscript{\rm 1},
    Armit Singh Bedi\textsuperscript{\rm 2},
    Vaneet Aggarwal\textsuperscript{\rm 1}
}
\affiliations {
	\textsuperscript{\rm 1} Purdue University\\
	\textsuperscript{\rm 2} University of Maryland\\
	bai113@purdue.edu, University of Maryland, vaneet@purdue.edu
}

\begin{document}

\maketitle

\begin{abstract}
    We consider the problem of constrained Markov decision process (CMDP) in continuous state-actions spaces where the goal is to maximize the expected cumulative reward subject to some constraints. We propose a novel Conservative Natural Policy Gradient Primal-Dual Algorithm (C-NPG-PD) to achieve zero constraint violation while achieving state of the art convergence results for the objective value function. For general policy parametrization, we prove convergence of value function to global optimal upto an approximation error due to restricted policy class. We even improve the sample complexity of existing constrained NPG-PD algorithm \cite{Ding2020} from $\mathcal{O}(1/\epsilon^6)$ to $\mathcal{O}(1/\epsilon^4)$. To the best of our knowledge, this is the first work to establish zero constraint violation with Natural policy gradient style algorithms for infinite horizon discounted CMDPs. We demonstrate the merits of proposed algorithm via experimental evaluations.
\end{abstract}
\section{Introduction}
Reinforcement learning problem is formulated as a Markov Decision Process (MDP) and can be solved using different algorithms in the literature \cite{sutton1988learning}. To deal with the scalability issue to the the large state and action spaces, policy parametrization is widely used \cite{Ding2020,xu2021crpo,Alekh2020}. The problem becomes challenging when we have constraints and is called  constrained MDPs (CMDPs). The problem is popular across various application domains such as robotics, communications, computer vision,  autonomous driving, etc. \cite{arulkumaran2017deep,kiran2021deep}. Mathematically, the problem is sequential in nature, agent observes the state, takes an action, and then transitions to next state. Further, an agent also needs to satisfy a set of constraints as well such as safety constraints, power constraints and maneuver constraints. CMDPs are challenging to solve specially in large state action spaces \cite{Ding2020,xu2021crpo} which is the focus of our work. 

The constraint violations could be catastrophic in applications such as in power systems \cite{vu2020safe} or autonomous vehicle control \cite{wen2020safe}. In the literature, various algorithms are proposed to solve CMDP in large actions spaces in a model free manner (See Table 1 for comparisons). The main performance metric here is the sample complexity, which is the number of samples required to achieve $\epsilon$-optimal objective and $\epsilon$-constraint violation. However, there doesn't exist literature which gives zero violation gurantee on large state and action space. Hence, we ask this question: ``{\emph{Is it possible to achieve zero constraint violations for CMDP problems in large state action spaces while solving in a model free manner?}}"


We answer this question in an affirmative sense in this work. We proposed a novel Conservative Natural Policy Gradient Primal Dual Algorithm (C-NPG-PDA) in this paper. We utilize a novel idea of conservative constraints to policy gradient algorithms and establish convergence guarantees of global optima for general policy parametrization. Our contributions are summarized as follows.
\begin{itemize}
    \item We propose a Natural Policy Gradient algorithm which achieves \textbf{zero constraint violation} for constrained MDPs in large state and action space. The proposed algorithm also converges to the neighborhood of the global optima with general parametrization. It is challenging to establish the zero violation result with general parametrization due to the lack of strong duality, and hence we perform a novel analysis to establish a bound between the conservative and original problems.
    
    \item We show that even if we don't utilize the conservative idea (proposed in this work), we are able to improve the sample complexity from $\ccalO\bigg(\frac{1}{\epsilon^6}\bigg)$ \cite{Ding2020}[Theorem 3] to $\ccalO\bigg(\frac{1}{\epsilon^4}\bigg)$. To achieve this, we utilize the first-order stationary result from \cite{Yanli2020} to bound the NPG update direction. However, due to the introduction of the constraint and the Lagrange function, the update of the Lagrange multiplier needs to be considered in the analysis. Moreover, we revise the analysis of the constraint violation in Theorem \ref{main_theorem}, which utilizes Lemma \ref{lem.constraint}, to further improve sample complexity of constraint violation.
    
    \item We perform initial proof of concepts experiments of the proposed algorithm with a random CMDP model and validate the convergence of the objective, with zero constraint violations.
    
\end{itemize} 

\begin{table*}[t]
	\centering
    \resizebox{0.99\textwidth}{!}{\begin{tabular}{|c|c|c|c|c|}
    \hline
    Parametrization            &   Algorithm                                           &                       Sample Complexity                                 & Constraint violation               & Generative Model \\
    \hline
    Softmax &   PMD-PD \cite{PMDPD}           & $\ccalO\bigg(1/\epsilon^3\bigg)$\footnote[1]  & Zero  & {No} \\
    \hline
                &  PD-NAC  \cite{PDNACA}     & $\ccalO\bigg(1/\epsilon^6\bigg)$\footnote[1]  &  Zero   & {No} \\
    \hline
                &  NPG-PD  \cite{Ding2020}                       & $\ccalO\bigg(1/(1-\gamma)^5\epsilon^2\bigg)$     & $\tdO(\epsilon)$   & {Yes}\\
    \hline
                &  CRPO \cite{xu2021crpo}                  & $\ccalO\bigg(1/(1-\gamma)^7\epsilon^4\bigg)$       & $\tdO(\epsilon)$                 & {Yes}\\
    \hline\hline
    General  &     NPG-PD \cite{Ding2020}                      &  $\ccalO\bigg(1/(1-\gamma)^8\epsilon^6\bigg)$    & $\tdO(\epsilon)$   & {Yes}\\
    \hline
                &   CRPO \cite{xu2021crpo}                             & $\ccalO\bigg(1/(1-\gamma)^{13}\epsilon^6\bigg)$ \footnote[2]      &  $\ccalO(\epsilon)$   & {Yes}\\
    \hline
               &   \textbf{C-NPG-PDA} (\textbf{This work}, Theorem \ref{main_theorem}) & $\tdO\bigg(1/(1-\gamma)^8\eps^4\bigg)$       & \textbf{Zero}    & {\textbf{No}}\\
    \hline\hline
    Lower bound &    \cite{Csaba2022lower}                  & $\tilde{\Omega}\bigg(1/(1-\gamma)^5\epsilon^2\bigg)$   & Zero                 & {N/A}\\
    \hline
\end{tabular}}
\caption{ This table summarizes the different state of the art policy-based algorithms available in the literature with softmax or general Parametrization for CMDPs. We note that the proposed algorithm in this work is able to achieve the best sample complexity among them all while achieving zero constraint violation as well. 
}
\label{tab:compare}
\end{table*}
\footnotetext[1] {The detailed dependence on $(1-\gamma)$ is not shown in the original paper.}
\footnotetext[2] {In \cite{xu2021crpo}, the authors used a two layer neural network with $m$ as the width of the neural network.  Larger width gives improved function approximation while increasing sample complexity. In its Theorem 2, if we choose $m=\ccalO(T^4)$, then it gives $\epsilon$-convergence to the global optima and the sample complexity is $T\cdot K_{in}=\ccalO(1/(1-\gamma)^13\eps^6)$. We note that this is the best choice (for sample complexity) that gives the error as $\epsilon$.}

\section{Related Work}
{\bf Policy Gradient for Reinforcement Learning: } Reinforcement Learning algorithms can be divided into policy-based or value-based algorithm. Thanks to the Policy Gradient Theorem \cite{sutton2000}, it is possible to obtain the gradient ascent direction for the standard reinforcement learning with the policy parameterization. However, in general, the objective in the reinforcement learning is non-convex with respective to the parameters \cite{Alekh2020}, which makes the theory of global convergence  difficult to derive and previous works \cite{Tianbing2017,Xu2019,xu2020} are focused on the first order convergence. Recently, there is a line of interest on the global convergence result for reinforcement learning. The authors in \cite{Kaiqing2019} apply the idea of escaping saddle points to the policy gradient and prove the convergence to the local optima. Further, authors in \cite{Alekh2020} provide provable global convergence result for direct parameterization and softmax parameterization with convergence rate $\ccalO(1/\sqrt{T})$ and sample complexity $\ccalO(1/\epsilon^6)$ in the tabular setting. For the restrictive policy parameterization setting, they propose a variant of NPG, Q-NPG and analyze the global convergence result with the function approximation error for both NPG and Q-NPG. \cite{Mei2020} improves the convergence rate for policy gradient with softmax parameterization from $\mathcal{O}(1/\sqrt{t})$ to $\mathcal{O}(1/t)$ and shows a significantly faster linear convergence rate $\mathcal{O}(\exp(-t))$ for the entropy regularized policy gradient. However, no sample complexity result is achievable because policy evaluation has not been considered. With actor-critic method \cite{konda2000actor}, \cite{Lingxiao2019} establishes the global optimal result for neural policy gradient method. \cite{Yanli2020} proposes a general framework of the analysis for policy gradient type of algorithms and gives the sample complexity for PG, NPG and the variance reduced version of them. 

{\bf Policy Gradient for Constrained Reinforcement Learning: } Although there exists quite a few studies for the un-constrained reinforcement learning problems, the research for the constrained setting is in its infancy and summarized in Table \ref{tab:compare}. The most famous method for the constrained problem is to use a primal-dual based algorithm. With the softmax-parametrization, \cite{PMDPD} proposed policy mirror descent-primal dual (PMD-PD) algorithm to achieve zero constraint violation and achieve $\ccalO(1/\epsilon^3)$ sample complexity. \cite{PDNACA} proposed an Online Primal-Dual Natural Actor-Critic Algorithm and achieves zero constraint violation with $\ccalO(1/\epsilon^6)$ sample complexity without the generative model.  \cite{Ding2020} proposed a primal-dual Natural Policy Gradient algorithm for both the softmax parametrization and general parametrization. However, the sample complexity for general case in their paper is $\ccalO(1/\epsilon^6)$ which is quite high. \cite{xu2021crpo} propose a primal approach policy-based algorithm for both the softmax parametrization and function approximation case. However, none of them achieve the zero constraint violation for the general parametrization case. As seen in Table \ref{tab:compare} we achieve the best result for sample complexity in CMDP with general parametrization while also achieving zero constraint violation.

\section{Problem  Formulation}\label{sec_formulation}
	We consider an infinite-horizon discounted Markov Decision Process $\mathcal{M}$ defined by the tuple $(\mathcal{S},\mathcal{A},\mathbb{P},r,g,\gamma,\rho)$, where $\mathcal{S}$ and $\mathcal{A}$ denote the state and action space, respectively. In this paper, we focus on large station and action space, which means that the policy parametrization may not be fully sufficient. $\mathbb{P}: \mathcal{S}\times\mathcal{A}\rightarrow [0,1]$ denotes the transition probability distribution from a state-action pair to another state. $r: \mathcal{S}\times\mathcal{A}\rightarrow \Delta^S$ denotes the reward for the agent and $g^i: \mathcal{S}\times\mathcal{A}\rightarrow [-1,1],i\in[I]$ defines the $i^{th}$ constraint function for the agent. $\gamma\in(0,1)$ is the discounted factor and $\rho: \mathcal{S}\rightarrow [0,1]$ is the initial state distribution.
	
	Define a joint stationary policy $\pi:\mathcal{S}\rightarrow\Delta^{\mathcal{A}}$ that maps a state $s\in\mathcal{S}$  to a probability distribution of actions defined as $\Delta^{\mathcal{A}}$  with a probability assigned to each action $a\in\mathcal{A}$. At the beginning of the MDP, an initial state $s_0\sim\rho$ is given and agent makes a decision $a_0\sim\pi(\cdot\vert s_0)$. The agent receives its reward $r(s_0,a_0)$ and constraints $g_i(s_0,a_0), i\in[I]$. Then it moves to a new state $s_1\sim\mathbb{P}(\cdot\vert s_0,a_0)$. We define the reward value function $J_r(\pi)$ and constraint value function $J_{g^i}(\pi), i\in[I]$ for the agent following policy $\pi$ as a discounted sum of reward and constraints over infinite horizon
	\begin{equation}\label{value_function}
	    \begin{aligned}
	        V_r^{\pi}(s)&=\bbE\bigg[\sum_{t=0}^{\infty}\gamma^t r(s_t,a_t)\bigg|s_0=s\bigg],\\
	        V_{g^i}^{\pi}(s)&=\bbE\bigg[\sum_{t=0}^{\infty}\gamma^t g^i(s_t,a_t)\bigg|s_0=s\bigg].
        \end{aligned}
    \end{equation}
    where $a_t\sim\pi(\cdot\vert s_t)$ and $s_{t+1}\sim\mathbb{P}(\cdot\vert s_t,a_t)$. Denote $J_r^{\pi}$ and $J_{g^i}^{\pi}$ as the expected value function w.r.t. the initial distribution such as
	\begin{equation}
		\begin{aligned}
		J_{r}(\pi)&=\mathbf{E}_{s_0\sim\rho}[V_r^{\pi}(s_0)], \\
		\text{and} \ \ J_{g^i}(\pi)&=\mathbf{E}_{s_0\sim\rho}[V_{g^i}^{\pi}(s_0)].
		\end{aligned}
	\end{equation} 
	 The agent aims to maximize the reward value function and satisfies constraints simultaneously. Formally, the problem can be formulated as
	\begin{equation}\label{eq:origin_problem}
		\begin{aligned}
		&\max_{\pi}\quad J_{r}(\pi)\\
		& s.t.\quad  J_{g^i}(\pi)\geq 0, \forall i\in[I].
		\end{aligned}
	\end{equation}
    Define $\pi^*$ as the optimal-policy for the above problem. Here, we introduce the Slater Condition, which means the above problem is strictly feasible.
    \begin{assumption}[Slater Condition]\label{ass_slater}
        There exists a $\varphi>0$ and $\bar{\pi}$ that $J_{g^i}(\bar{\pi})\geq \varphi, \forall i\in[I]$.
    \end{assumption}

\section{Proposed Approach}
    We consider a policy-based algorithm on this problem and parameterize the policy $\pi$ as $\pi_\theta$ for some parameter $\theta\in\Theta$ such as softmax parametrization or a deep neural network. 
    In this section, we first give the form of the true gradient and introduce some properties of it. Then, we propose the Conservative Natural Policy Descent Primal-Dual Algorithm (C-NPG-PD), where the conservative idea is utilized to achieve zero constraint violation.
	\subsection{Gradient of Value Function and Properties}
	For the analysis of the convergence for the proposed algorithm, it is necessary to establish the form of the true and its properties. Here, we utilize the Policy Gradient Theorem and write the gradient for the objective function as
	\begin{equation}\label{eq:gradient_simplified}
		\begin{aligned}
			&\nabla_\theta J_r(\pi_\theta)=\\
			&\mathbf{E}_{\tau\sim p(\tau\vert \theta)}\bigg[\sum_{t=0}^\infty \nabla_\theta \log(\pi_\theta(a_t\vert s_t))\bigg(\sum_{h=t}^\infty \gamma^hr(s_h,a_h)\bigg)\bigg]
		\end{aligned}
	\end{equation}
	The computation of the gradient is well known and the proof is removed to the Appendix for completeness. We note that the log-policy function $\log\pi_\theta(a\vert s)$ is also called log-likelihood function in statistics \cite{Kay97} and we make the following assumption.
	\begin{assumption}\label{ass_score}
		The log-likelihood function is $G$-Lipschitz and $M$-smooth. Formally,
		\begin{equation}
	        \begin{aligned}
	        \Vert \nabla_\theta\log\pi_\theta(a\vert s)\Vert\leq G\quad\forall \theta\in\Theta,\forall (s,a)\in\mathcal{S}\times\mathcal{A},\\
    		\Vert \nabla_\theta\log\pi_{\theta_1}(a\vert s)-\nabla_\theta\log\pi_{\theta_2}(a\vert s)\Vert\leq M\Vert \theta_1-\theta_2\Vert\\
    		\forall \theta_1,\theta_2 \in\Theta,\forall (s,a)\in\mathcal{S}\times\mathcal{A}.
	        \end{aligned}
		\end{equation}
	\end{assumption}
	\begin{remark}
		The Lipschitz and smoothness properties for the log-likelihood are quite common in the field of policy gradient algorithm \cite{Alekh2020,Mengdi2021,Yanli2020}. Such properties can also be verified for simple parametrization such as Gaussian policy.   
	\end{remark}
	The following two lemmas give the property of the value functions and its gradient, which are useful in the convergence proof. The detailed proof can be found in Appendix.
	\footnote{The appendix is uploaded to \url{https://arxiv.org/abs/2206.05850}}
	\begin{lemma}\label{lem_smooth}
		Under Assumption \ref{ass_score}, both the objective function $\bm{J}_r^{\pi_\theta}$ and the constraint function $\bm{J}_{g^i}^{\pi_\theta}$  are $L_J$-smooth w.r.t. $\theta$. Formally,
		\begin{equation}
		    \Vert\nabla_\theta \bm{J}_r(\theta_1)-\nabla_\theta \bm{J}_r(\theta_2)\Vert_2\leq L_J\Vert\theta_1-\theta_2\Vert_2 \quad\forall \theta_1,\theta_2\in \Theta
		\end{equation}
		where $L_J=\frac{M}{(1-\gamma)^2}+\frac{2G^2}{(1-\gamma)^3}$
	\end{lemma}
	\begin{lemma}\label{lem_grad_bound}
		Under Assumption \ref{ass_score}, both the gradient of objective function $\nabla_\theta \bm{J}_r^{\pi_\theta}$ and that of the constraint function $\nabla_\theta \bm{J}_{g^i}^{\pi_\theta}$ are bounded. Formally,
		\begin{equation*}
			\begin{aligned}
				&\Vert\nabla_\theta J_r(\theta)\Vert_2\leq \frac{G}{(1-\gamma)^2}\\
				&\Vert\nabla_\theta J_{g^i}(\theta)\Vert_2\leq \frac{G}{(1-\gamma)^2}\quad \forall i\in[I].
			\end{aligned}
		\end{equation*}
	\end{lemma}
	\subsection{Natural Policy Gradient Primal-Dual Method with Zero Constraint Violation}
	In order to achieve zero constraint violation, we consider the conservative stochastic optimization framework proposed in \cite{Amrit_zero} and define the conservative version of the original problem as
	\begin{equation}\label{eq:conservative_problem}
		\begin{aligned}
		&\max_{\pi}\quad J_{r}(\pi)\\
		& s.t.\quad  J_{g^i}(\pi)\geq \kappa, \forall i\in[I]
		\end{aligned}
	\end{equation}
	where $\kappa>0$ is the parameter to control the constraint violation which we will explicitly mention in Theorem \ref{main_theorem} in Sec. \ref{theorem_sub}. The idea here to achieve zero constraint violation is to consider the tighter problem to make it less possible to make violation for the original problem. Notice that it is obvious that $\kappa$ must be less than $\frac{1}{1-\gamma}$ to make the conservative problem still feasible. Combining this idea, we introduce the Natural Policy Gradient method.
	The NPG Method utilizes the Fisher information matrix defined as
	\begin{equation}
		F_\rho(\theta)=\mathbf{E}_{s\sim d_\rho^{\pi_\theta}}\mathbf{E}_{a\sim\pi_\theta}[\nabla_{\theta}\log\pi_\theta(a|s)\nabla_\theta\log\pi_\theta(a|s)^T]
	\end{equation}
	where $d_\rho^\pi$ is the state visitation measure defined as 
	\begin{equation}
	    d_{\rho}^{\pi}:=(1-\gamma)\mathbf{E}_{s_0\sim\rho}\bigg[\sum_{t=0}^{\infty}\gamma^t\mathbf{Pr}^{\pi}(s_t=s\vert s_0)\bigg]
	\end{equation}
		We define the Lagrange function as 
	\begin{equation}
        J_L(\pi_\theta,\bblambda)=J_r(\pi_\theta)+\sum_{i\in[I]}\lambda^i(\pi_\theta) J_{g^i}(\pi_\theta)
	\end{equation}
	For simplicity, we denote $J_r(\theta),J_{g^i}(\theta), J_L(\theta,\bblambda)$ as the short for $J_r(\pi_\theta),J_{g^i}(\pi_\theta), J_L(\pi_\theta,\bblambda)$ and the Natural Policy Gradient method is written as 
	\begin{equation}
		\begin{aligned}
		&\theta^{k+1}=\theta^{k}+\eta_1 F_\rho(\theta^k)^\dagger\nabla_\theta J_L(\theta^t,\bblambda^k)\\
		&\lambda_i^{k+1}=\ccalP_{(0,\Lambda]}\bigg(\lambda_i^k-\eta_2 \big(J_{g^i}(\theta^k)-\kappa\big)\bigg)
		\end{aligned}
	\end{equation}
    where $\Lambda=\frac{2}{(1-\gamma)\varphi}$. It is proved in Lemma \ref{lem.boundness} that the optimal dual variable is bounded in $(0,\Lambda]$. We note that the pseudo-inverse of the Fisher information matrix is difficult to calculate. However, the NPG update direction can be related to the compatible function approximation error defined as 
	\begin{equation}
		\begin{aligned}
			&L_{d_\rho^{\pi},\pi}(\omega,\theta,\bblambda)=\mathbf{E}_{s\sim d_\rho^{\pi}}\mathbf{E}_{a\sim\pi(\cdot\vert s)}\\
            &\bigg[\bigg(\nabla_\theta\log\pi_{\theta}(a\vert s)\cdot(1-\gamma)\omega-A_{L,\bblambda}^{\pi_\theta}(s,a)\bigg)^2\bigg]
		\end{aligned}
	\end{equation}
	Given a fixed $\bblambda^k$ and $\theta^k$, it can be proved that the minimizer $\omega_*^k$ of $L_{d_\rho^{\pi},\pi}(\omega,\theta^k,\bblambda^k)$ is exact the NPG update direction (see Lemma \ref{lem:NPG_direction}). Thus, it is possible to utilize the Stochastic Gradient Descent (SGD) algorithm to achieve the minimizer $\omega_*^k$. The gradient of $L_{d_\rho^{\pi},\pi}(\omega,\theta^k,\bblambda^k)$ can be computed as 
	\begin{equation}\label{eq:gradient_compatible}
		\begin{aligned}
			\nabla_\omega &L_{d_\rho^{\pi},\pi}(\omega,\theta^k,\bblambda^k) = 2(1-\gamma)\nabla_\theta\log\pi_\theta^k(a\vert s)\cdot\\
			&\mathbf{E}_{s\sim d_\rho^{\pi}}\mathbf{E}_{a\sim\pi(\cdot\vert s)}\bigg[\nabla_\theta\log\pi_\theta^k(a\vert s)\cdot(1-\gamma)\omega-A_{L,\bblambda^k}^{\pi_\theta^k}(s,a)\bigg]
		\end{aligned}
	\end{equation}
	Where $A_{L,\bblambda^k}^{\pi_\theta^k}(s,a)$ is the advantage function for the Lagrange function and is defined as
	\begin{equation}\label{eq:define_advantage}
		\begin{aligned}
		    A_{L,\bblambda^k}^{\pi_\theta^k}(s,a)&=\bigg[Q_r^{\pi_\theta^k}(s,a)-V_r^{\pi_\theta^k}(s)\bigg]\\
		    &+\sum_{i\in[I]}\lambda_k^i\bigg[ Q_{g^i}^{\pi_\theta^k}(s,a)-V_{g^i}^{\pi_\theta^k}(s)\bigg]
		\end{aligned}
	\end{equation}
	However, it is challenging to achieve the exact value of the advantage function and thus we estimate it as $\hat{A}_{L,\bblambda^k}^{\pi_\theta^k}(s,a)$ using the following procedure.
    The stochastic version of gradient can be written as
	\begin{equation}\label{eq:SGD_grad_est}
		\begin{aligned}
			\hat{\nabla}_\omega &L_{d_\rho^{\pi},\pi}(\omega,\theta^k,\bblambda^k)= 2(1-\gamma)\nabla_\theta\log\pi_\theta^k(a\vert s)\cdot\\
			&\bigg[\nabla_\theta\log\pi_\theta^k(a\vert s)\cdot(1-\gamma)\omega-\hat{A}_{L,\bblambda^k}^{\pi_\theta^k}(s,a)\bigg] 
		\end{aligned}
	\end{equation}
	Based on the stochastic version of the gradient mentioned above, we propose the Natural Gradient Descent Primal Dual with Zero Violation in Algorithm \ref{alg:spdgd}. In line 1, we initialize the parameter $\theta$ and Lagrange multiplier $\bblambda$. From Line 3 to Line 10, we use SGD to compute the Natural Policy gradient. From Line 11 to Line 15, we estimate an unbiased value function for constraint. Finally, in Line 16, we perform the conservative primal-dual update.
	\begin{algorithm}[t]
	    \caption{\textbf{C}onservative \textbf{N}atural \textbf{G}radient \textbf{D}escent \textbf{P}rimal-\textbf{D}ual \textbf{A}lgorithm (C-NPG-PDA)}
	    \label{alg:spdgd}
	    \textbf{Input}: Sample size K, SGD learning iteration $N$, Initial distribution $\bbrho$. Discounted factor $\gamma$.\\
	    \textbf{Parameter}: Step-size $\eta_1$, $\eta_2$, SGD learning rate $\alpha$, Slater variable $\varphi$, Conservative variable $\kappa$\\
	    \textbf{Output}: $\bbarlambda=\frac{1}{T}\sum_{t=1}^T\bblambda^t$, $\bbaru=\frac{1}{T}\sum_{t=1}^T\bbu^t$ and $\bbarv=\frac{1}{T}\sum_{t=1}^T\bbv^t$
	    \begin{algorithmic}[1] 
		    \STATE Initialize $\bblambda^1=\bbzero$, $\theta^1=0$, $\omega_0=0$
		    \FOR{$k=1,2,...,K$} 
		        \FOR{$n=1,2,...,N$} 
		            \STATE Sample $s\sim d_\rho^{\pi_{\theta^k}}$ and $a\sim \pi_{\theta^k}(\cdot|s)$ 
		            \STATE Sample $Q^{\pi_\theta^k}(s,a)$ and $V^{\pi_\theta^k}(s)$ for reward function and constraint functions following Algorithm \ref{alg:sample}
		            \STATE Estimate the Advantage Function $\hat{A}_{L,\bblambda^k}^{\pi_\theta^k}(s,a)$ following Eq. \eqref{eq:define_advantage}
		            \STATE Estimate SGD gradient $\hat{\nabla}_\omega L_{d_\rho^{\pi},\pi}(\omega_{n},\theta^k,\bblambda^k)$ following Eq. \eqref{eq:SGD_grad_est}
		            \STATE SGD update $\omega_{n+1}=\omega_{n}-\alpha\cdot \hat{\nabla}_\omega L_{d_\rho^{\pi},\pi}(\omega_{n},\theta^k,\bblambda^k)$ 
		        \ENDFOR
		        \STATE Compute NPG update direction as $\omega = \frac{1}{N}\sum_{n=1}^{N}\omega_n$
		        \FOR{$n=1,2,...,N$} 
		            \STATE Sample $s\sim \rho$ and $a\sim \pi_{\theta^k}(\cdot|s)$ 
		            \STATE Sample constraint value functions $V_{g^i,n}^{\pi_{\theta^k}}(s)$  following Algorithm \ref{alg:sample}
		        \ENDFOR
		        \STATE Estimate expected constraint value function $\hat{J}_{g^i}(\pi_\theta^k)=\frac{1}{N}\sum_{n=1}^{N}V_{g^i,n}^{\pi_{\theta^k}},\forall i\in[I]$
		        \STATE Update the primal and dual variable as
		        \begin{align}\label{eq:update_lambda1}
			        \theta^{k+1}&=\theta^k+\eta_1\omega\\
			         \lambda_i^{k+1}&=\ccalP_{(0,\Lambda]}\bigg(\lambda_i^{k}-\eta_2(\hat{J}_{g^i}(\pi_{\theta^k})-\kappa)\bigg)\label{eq:update_lambda2},\forall i\in[I]
		        \end{align}
		\ENDFOR
	    \end{algorithmic}
    \end{algorithm}
    \begin{algorithm}[t]
	    \caption{Estimate Value Function for objective or constraint function}
	    \label{alg:sample}
	    \textbf{Input}: starting state and action $s,a$, reward function $r$ or constraint function $g^i$ (Here we denote as function $h$ for simplicity), policy $\pi$, discounted factor $\gamma$, Access to Generative model $CMDP(\ccalS,\ccalA,P,r,g^i,\rho,\gamma)$\\
	    \textbf{Output}: state action value function $\hat{Q_H}(s,a)$ or state value function $\hat{V_h}(s)$
	    \begin{algorithmic}[1] 
		    \STATE Estimate state action value function as 
		    $\hat{Q_h}(s,a)=\sum_{t=0}^{T-1}h(s_t,a_t)$, where $s_0=s$,$a_0=a$,$a_t\sim \pi(\cdot|s_t)$, $s_{t+1}\sim P(\cdot|s_t, a_t)$,$T\sim Geo(1-\gamma)$
		    \STATE Estimate state value function as 
		    $\hat{V_h}(s)=\sum_{t=0}^{T-1}h(s_t,a_t)$, where $s_0=s$,$a_t\sim \pi(\cdot|s_t)$,$s_{t+1}\sim P(\cdot|s_t, a_t)$,$T\sim Geo(1-\gamma)$
	    \end{algorithmic}
    \end{algorithm}

\section{Convergence Rate Results}\label{sec_result}
	Before stating the convergence result for the policy gradient algorithm, we describe the following assumptions which will be needed for the main result. 
	\begin{assumption}\label{ass_pd}
		For all $\theta\in\mathbb{R}^d$, the Fisher information matrix induced by policy $\pi_{\theta}$ and initial state distribution $\rho$ satisfies
		\begin{equation}
            \begin{aligned}
                F_\rho(\theta)&=\mathbf{E}_{s\sim d_\rho^{\pi_\theta}}\mathbf{E}_{a\sim\pi_\theta}[\nabla_{\theta}\log\pi_\theta(a|s)\nabla_\theta\log\pi_\theta(a|s)^T]\\
                &\succeq\mu_F\cdot \mathbf{I}_d
            \end{aligned}
		\end{equation}
		for some constant $\mu_F>0$
	\end{assumption}
	\begin{remark}
		The positive definiteness assumption is standard in the field of policy gradient-based algorithms \cite{Kakade2002,Peters2008,Yanli2020,Kaiqing2019}. A common example that satisfies such an assumption is Gaussian policy with mean parameterized linearly (See Appendix B.2 in \cite{Yanli2020}).
	\end{remark}
	\begin{assumption}\label{ass_transfer_error}
		Define the transferred function approximation error as below
		\begin{equation}\label{eq:transfer_error}
			\begin{aligned}
				&L_{d_\rho^{\pi^*},\pi^*}(\omega,\theta,\bblambda)=\\
				&\mathbf{E}_{s\sim d_\rho^{\pi^*}}\mathbf{E}_{a\sim\pi^*}\bigg[\bigg(\nabla_\theta\log\pi_\theta(a\vert s)\cdot(1-\gamma)\omega-A_{L,\bblambda}^{\pi_\theta}(s,a)\bigg)^2\bigg]
			\end{aligned}
		\end{equation}
		We assume that this error satisfies $L_{d_\rho^{\pi^*},\pi^*}(\omega_*^{\theta,\bblambda},\theta,\bblambda)\leq \epsilon_{bias}$ for any $\theta\in\Theta,\bblambda\in\Lambda$, where $\omega_*^{\theta,\bblambda}$ is given as
		\begin{equation}\label{eq:NPG_direction}
			\begin{aligned}
				&\omega_*^{\theta,\bblambda}=\arg\min_{\omega}L_{d_\rho^{\pi^\theta},\pi^\theta}(\omega,\theta,\bblambda)=\arg\min_{\omega}\\
				&\mathbf{E}_{s\sim d_\rho^{\pi_{\theta}}}\mathbf{E}_{a\sim\pi_{\theta}}\bigg[\bigg(\nabla_\theta\log\pi_\theta(a\vert s)\cdot(1-\gamma)\omega-A_{L,\bblambda}^{\pi_\theta}(s,a)\bigg)^2\bigg]
			\end{aligned}
		\end{equation}
		It can be shown that $\omega_*^\theta$ is the exact Natural Policy Gradient (NPG) update direction.
	\end{assumption}
	\begin{remark}
		By  Eq. \eqref{eq:transfer_error} and \eqref{eq:NPG_direction}, the transferred function approximation error expresses an approximation error with distribution shifted to $(d_\rho^{\pi^*},\pi^*)$. With the softmax parameterization or linear MDP structure \cite{Chi2019}, it has been shown that $\epsilon_{bias}=0$ \cite{Alekh2020}. When parameterized by the restricted policy class, $\epsilon>0$ due to $\pi_\theta$ not containing all policies. However, for a rich neural network parameterization, the $\epsilon_{bias}$ is small \cite{Lingxiao2019}. A similar assumption has been adopted in \cite{Yanli2020} and \cite{Alekh2020}. 
	\end{remark}
	
	\subsection{Global Convergence For NPG-PD Method} \label{theorem_sub}
	To analyze the global convergence of the proposed algorithm, we first demonstrate the convergence of the Lagrange function for the conservative problem, which is shown in the following Lemma.
	\begin{lemma}\label{lem_framework}
		Suppose a general primal-dual gradient ascent algorithm updates the parameter as
		\begin{equation}
			\begin{aligned}
				&\theta^{k+1}=\theta^k+\eta\omega^k\\
				&\lambda_i^{k+1}=\ccalP_{(0,\Lambda]}\bigg(\lambda_i^k-\eta_2 \big(J_{g^i}(\theta^k)-\kappa\big)\bigg) 
			\end{aligned}
		\end{equation}
		When Assumptions \ref{ass_score} and \ref{ass_transfer_error} hold, we have
		\begin{equation}
			\begin{aligned}
				&\frac{1}{K}\sum_{k=1}^{K}\mathbf{E}\bigg(J_L(\pi^*_{\theta,\kappa},\bblambda^k)-J_L(\pi_\theta^k,\bblambda^k)\bigg)\leq \\
				&\frac{\sqrt{\epsilon_{bias}}}{1-\gamma}+\frac{M\eta_1}{2K}\sum_{k=0}^{K-1}\mathbf{E}\Vert\omega^k\Vert^2+\frac{\log(|\ccalA|)}{\eta_1 K}\\
				&+\frac{G}{K}\sum_{k=1}^{K}\mathbf{E}\Vert(\omega^k-\omega_*^k)\Vert_2
			\end{aligned}
		\end{equation}
		where $\omega_*^k:=\omega_*^{\theta^k}$ and is defined in Eq. \eqref{eq:NPG_direction}
	\end{lemma}
	To prove the above Lemma, we extend the result in \cite{Yanli2020}[Proposition 4.5] to our setting. The extended result is stated and proved in Lemma \ref{lem_framework} in Appendix \ref{sec_app_framwork}. Then, to prove the global convergence of the Lagrange function, it is sufficient to bound $\frac{G}{K}\sum_{k=1}^{K}\mathbf{E}\Vert(\omega^k-\omega_*^k)\Vert_2$ and $\frac{M\eta_1}{2K}\sum_{k=0}^{K-1}\mathbf{E}\Vert\omega^k\Vert^2$ in Lemma \ref{lem_framework}. The detailed proof of them can be found in Appendix \ref{sec:bound1} and \ref{sec:bound2}. At a high level, the first term is the difference between the estimated and exact NPG update direction, which can be bounded using the convergence of the SGD procedure. The second term is the bound of the norm of the estimated gradient. To bound the second term, we need the following first-order convergence result. 
    \begin{lemma}\label{lem:first_order}
       In the NPG update process, if we take $\eta_1=\frac{\mu_F^2}{4G^2L_J}$ and $\eta_2=\frac{1}{\sqrt{K}}$, for any given $\eps>0$, Let $K=\ccalO\bigg(\frac{I^2}{(1-\gamma)^4\epsilon^2}\bigg)$ and $N=\ccalO\bigg(\frac{I^2\Lambda^2}{(1-\gamma)^4\epsilon}\bigg)$, we have the convergence of first order stationary,
        \begin{equation}
	        \frac{1}{K}\sum_{k=0}^{K-1}\mathbf{E}\Vert \nabla_\theta J_L(\theta^k,\bblambda^k)\Vert_2^2\leq \epsilon
    	\end{equation}
    \end{lemma}
    \begin{remark}
        The basic idea of the proof for first-order stationary is from \cite{Yanli2020}. However, due to the introduction of the constraints, we need to further consider the update of the dual variable. The detailed proof can be found in Appendix \ref{sec:first_order}.
    \end{remark}
    Given the above Lemmas, it is sufficient to achieve the final bound of the Lagrange function for the conservative problem as below.
    \begin{lemma}\label{lem_bound_JL} 
	    Under the Assumption \ref{ass_score}, \ref{ass_pd} and \ref{ass_transfer_error}, if we take $\eta_1=\frac{\mu_F^2}{4G^2L_J}$ and $\eta_2=\frac{1}{\sqrt{K}}$, the proposed algorithm achieves the global convergence of the Lagrange function, which can be formally written as
	    \begin{equation}\label{eq:general_bound}
	        \frac{1}{K}\sum_{k=1}^{K}\mathbf{E}\bigg(J_L(\pi^*_{\theta,\kappa},\bblambda^k)-J_L(\pi_\theta^k,\bblambda^k)\bigg)\leq\frac{\sqrt{\epsilon_{bias}}}{1-\gamma}+\epsilon_{K,N}
	    \end{equation}
		where
        \begin{equation}
        	\begin{aligned}
                \eps_{K,N}&=\ccalO\bigg(\frac{1}{(1-\gamma)^3K}\bigg)+\ccalO\bigg(\frac{I^2\Lambda^2}{ (1-\gamma)^2 N}\bigg)\\
                &+\ccalO\bigg(\frac{I\Lambda}{(1-\gamma)\sqrt{N}}\bigg)+\ccalO\bigg(\frac{I\Lambda}{K(1-\gamma)}\bigg)\\
                &+\ccalO\bigg(\frac{I}{\sqrt{K}(1-\gamma)}\bigg)
        	\end{aligned}
        \end{equation}
	\end{lemma}
       Before we get the final result for the regret and constraint violation, we need to bound the gap between the optimal value function of the original problem and the conservative problem. Such a gap can be bounded in the dual domain. To do that, we recall the definition of state-action occupancy measure $d^{\pi}\in\mathbb{R}^{|S||A|}$ as
    \begin{equation}
	    d^{\pi}(s,a)=(1-\gamma)\mathbb{P}\bigg(\sum_{t=0}^{\infty}\gamma^t\cdot1_{s_t=s,a_t=a}|\pi,s_0\sim\rho\bigg)
    \end{equation}
    We note that the objective and constraint can be written as
    \begin{equation}\label{eq:new_value_function}
    	\begin{aligned}
    		&J_r(\pi_\theta)=\frac{1}{1-\gamma}\left<r, d^{\pi_\theta}\right>\\
    		&J_{g^i}(\pi_\theta)=\frac{1}{1-\gamma}\left<g^i, d^{\pi_\theta}\right>,\forall i\in[I]
    	\end{aligned}
    \end{equation}
    Define $\ccalD$ to be the set of vector $\phi\in \mathbb{R}^{\ccalS\times\ccalA}$ satisfying
    \begin{equation}\label{eq:occupancy_set}
    \left\{
        \begin{aligned}
        &\sum_{s'\in\ccalS}\sum_{a\in\ccalA}\phi(s',a)(\delta_s(s')-\gamma \mathbf{P}(s|s',a))=(1-\gamma)\rho(s)\\
        &\phi(s,a)\geq 0,\forall (s,a)\in \ccalS\times\ccalA
        \end{aligned}
    \right.
    \end{equation}
    By summing the first constraint over $s$, we have $\sum_{s,a}\phi(s,a)=1$, which means that $\phi$ in the above set is an occupancy measure. By Eq. \eqref{eq:new_value_function} and \eqref{eq:occupancy_set}, we define the following problem which can be found in the reference \cite{altman1999constrained}
    \begin{equation}\label{eq:new_problem}
        \begin{aligned}
        &\max_{\phi\in \ccalD} \quad \frac{1}{1-\gamma}\left<r, \phi\right>\\
        &s.t. \quad \frac{1}{1-\gamma}\left<g^i, \phi\right>\geq 0,\forall i\in[I]
        \end{aligned}
    \end{equation}
    For the full-parameterized policy, it can be shown that the above problem is equivalent to the original problem Eq. \eqref{eq:origin_problem}. However, the strong duality doesn't hold for general parameterization. Thus, we need the following assumption to bridge the gap between them.
    \begin{assumption}\label{ass:sufficient_para}
    	For any $\phi\in \ccalD$, we define a stationary policy as 
    	\begin{equation}\label{eq:stationary1}
    	    \pi'(a|s)=\frac{\phi(s,a)}{\sum_a \phi(s,a)}.
    	\end{equation}
    	We assume that there always exists a $\theta\in \Theta$ such that $|\pi'(a|s)-\pi_\theta(a|s)|\leq \epsilon_{bias2}, \forall (s,a)\in \ccalS\times\ccalA$
    \end{assumption}
    \begin{remark}
        The intuition behind the above assumption is that the parameterization is rich enough so that we can always find a certain parameter $\theta$ and $\pi_\theta$ is close to the above stationary policy. A special case is a softmax parameterization, where $\eps_{bias2}=0$.
    \end{remark}
    With such an assumption, we reveal the relationship between the optimal value of the primal problem and dual problem as follows, whose proof can be found in the Appendix.
    \begin{lemma}\label{lem:duality}
        Under Assumption \ref{ass:sufficient_para}, denote $\pi_{\theta^*}$ as the optimal policy of the original problem defined in Eq.\eqref{eq:origin_problem} and $\phi^*$ as the optimal occupancy measure for the new problem defined in Eq. \eqref{eq:new_problem}, we have
        \begin{equation}
	        \left<r, \phi^*\right>-\eps_{bias2}\leq J_r^{\pi_\theta^*}\leq \left<r, \phi^*\right>
        \end{equation}
    \end{lemma}
    Equipped with the above lemma, we bound the gap between the original problem and the conservative problem in the following lemma.
    \begin{lemma}\label{lem:bound_conservative}
        Under Assumption \ref{ass:sufficient_para}, Denote $\pi_{\theta_\kappa^*}$ as the optimal policy for the conservative problem, we have
        \begin{equation}
	        J_r^{\pi_{\theta^*}}-J_r^{\pi_{\theta_\kappa^*}}\leq \frac{\eps_{bias2}}{(1-\gamma)^2}+\frac{\kappa}{(1-\gamma)\varphi}
	    \end{equation}
    \end{lemma}
        
	Equipped with Lemma \ref{lem_bound_JL} and \ref{lem:bound_conservative}, we provide the main result for the NPG-PD algorithm for the objective function and constrained violation. The detailed proof can be found in the Appendix.
    \begin{figure*}[htbp]
	\centering
	\subfigure[Comparison for the convergence of objective]{
		\includegraphics[width=2.5in]{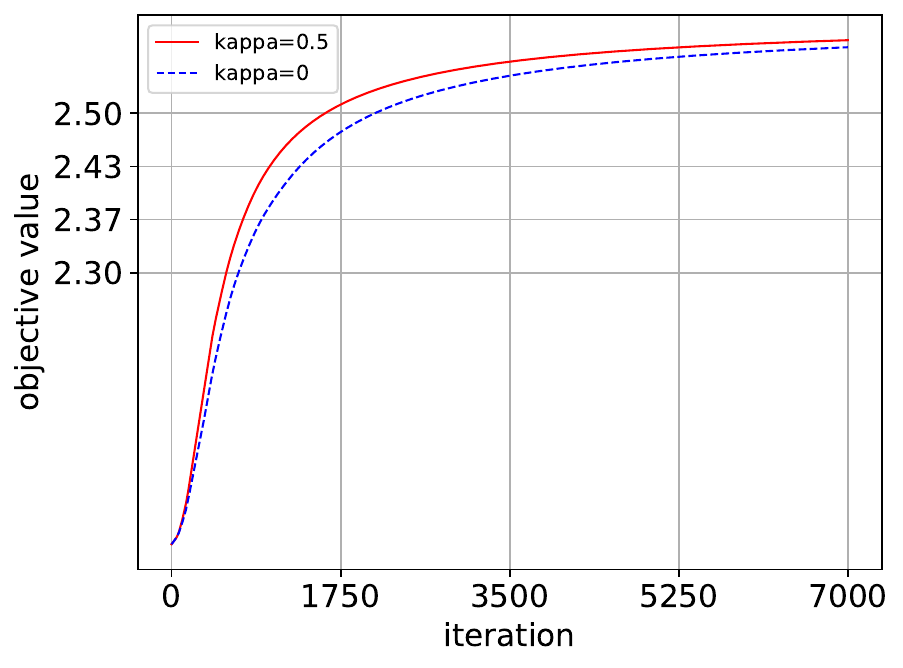}
	}
	\subfigure[Comparison for the constriant violation]{
		\includegraphics[width=2.5in]{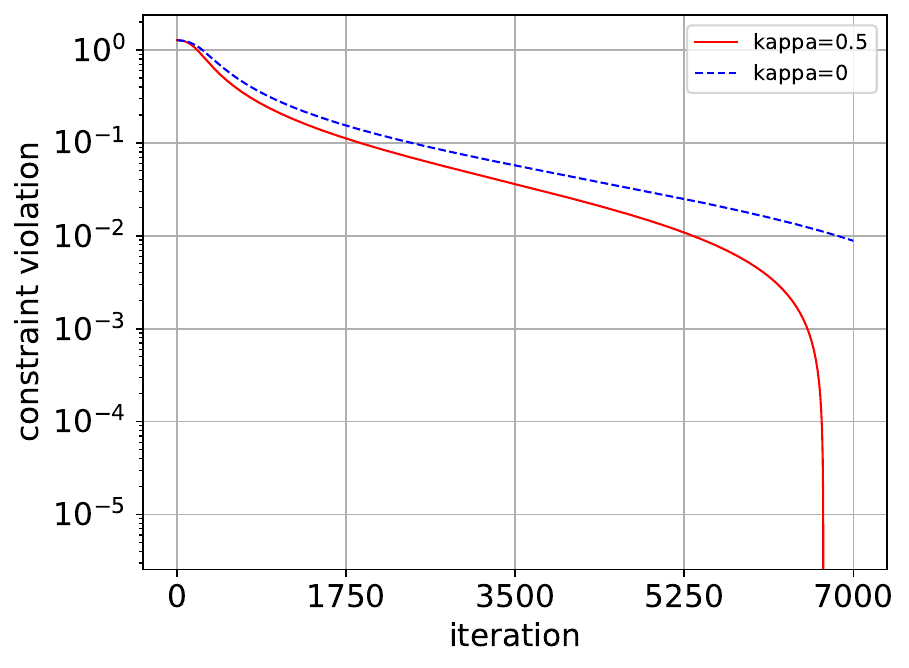}
	}
	\caption{Comparison of objective and constraint violation between $\kappa=0.5$ and $\kappa=0$. For the constraint violation figure, we use the log axis to make zero constraint violation more obvious.}
	\label{fig:compare}
\end{figure*}
	\begin{theorem}\label{main_theorem}
		For any $\epsilon>0$, in the Natural Policy Gradient Algorithm \ref{alg:spdgd}, if step-size $\eta_1=\frac{\mu_F^2}{4G^2L_J}$ and $\eta_2=\frac{1}{\sqrt{K}}$, the number of iterations $K=\ccalO\bigg(\frac{I^2}{\varphi^2(1-\gamma)^4\eps^2}\bigg)$, the number of samples for per iteration $N=\ccalO\bigg(\frac{I^2\Lambda^2}{(1-\gamma)^2\eps^2}\bigg)$ and take the conservative variable as 
		\begin{equation}
		    \kappa = \varphi\sqrt{\epsilon_{bias}}+ \varphi(1-\gamma)\epsilon_{K,N}  + \frac{2I}{(1-\gamma)\varphi\sqrt{K}} + \frac{2\varphi I}{\sqrt{K}(1-\gamma)}
		\end{equation}
		then we have $\epsilon$-optimal policy with zero constraint violations. Formally,
		\begin{equation}
		    \begin{aligned}
		        \frac{1}{K}\sum_{k=0}^{K-1}\bigg(J_r(\pi^*_{\theta})-J_r(\pi_\theta^k)\bigg)&\leq \ccalO\bigg(\frac{\sqrt{\epsilon_{bias}}}{1-\gamma}\bigg)+\ccalO\bigg(\frac{\eps_{bias2}}{(1-\gamma)^2}\bigg)\\
		        &+\ccalO\bigg(\epsilon\bigg)  \\ \frac{1}{K}\sum_{k=0}^{K-1}J_{g^i}(\pi_\theta^k)&\geq 0
		    \end{aligned}
		\end{equation}
		In other words, the NPG-PD algorithm needs $\mathcal{O}\big(\frac{I^4\Lambda^2}{\varphi^2(1-\gamma)^6\epsilon^4}\big)$ trajectories.
	\end{theorem}
    \begin{remark}
        The proposed algorithm doesn't only achieve the zero constraint violation but also achieves the state of art sample complexity over a general parameterization policy-based algorithm. In Theorem \ref{main_theorem}, we can see that the algorithm converges to the neighborhood of the global optimal and the bias is controlled by two parameters $\eps_{bias}$ and $\eps_{bias2}$ defined in Assumption \ref{ass_transfer_error} and \ref{ass:sufficient_para}, respectively. If the parameterization is sufficient enough, then $\eps_{bias}=\eps_{bias2}=0$. However, whether there exists a certain relationship between Assumption \ref{ass_transfer_error} and \ref{ass:sufficient_para} is an interesting question for future work.
    \end{remark}
\section{Simulation}

In order to verify the performance of the proposed algorithm (Algorithm \ref{alg:spdgd}), we utilize the simulation code from \cite{Ding2020} and compare the proposed algorithm with them. We establish a random CMDP, where the state space and action space are $|\ccalS|=10,|\ccalA|=5$. The transition matrix $P(s'|s,a)$ is chosen by generating each entry uniformly at random in $[0,1]$, followed by normalization. Similarly, the reward function $r(s,a)\sim U(0,1)$ and constraint function $g(s,a)\sim U(-0.71, 0.29)$. Only 1 constraint function is considered here. The initial state distribution is set to uniform and the discount factor is $\gamma=0.8$. For the general parameterization, we use a feature map with dimension $d=35$, and for each SGD procedure we use $N=100$ number of samples. The learning rate for $\theta$ and $\lambda$ are set to 0.1. The more detailed information for the simulation setting can be found in Appendix \ref{app:sim}. We run the algorithm for $K=7000$ iterations and compare the proposed algorithm with $\kappa=0.5$ and the NPG-PD algorithm \cite{Ding2020} which doesn't consider the zero constraint violation case (equivalently $\kappa=0$) in  Figure \ref{fig:compare}.

From Fig. \ref{fig:compare}, we find that the convergence of the reward is similar and the proposed algorithm converges even faster than the non-zero constraint violation case. However, for the constraint violation, we find that when $\kappa=0.5$,  the log of constraint violation converges to negative infinity, which means that the constraint violation is below 0. In contrast, the constraint violation still exists when $\kappa=0$. The comparison between $\kappa=0.5$ and $\kappa=0$ validates the result in Theorem \ref{main_theorem}.
\section{Conclusion}\label{sec_conclusion}
In this paper, we propose a novel algorithm for Constrained Markov Decision Process and the proposed algorithm achieves the state-of-the-art sample complexity over general parametrization policy-based algorithms. By revealing the relationship between the primal and dual problem, the gap between conservative problem and original problem is bounded, which finally leads to the analysis of zero constraint violation. The proposed algorithm converges to the neighbourhood of the global optimal and the gap is controlled by the richness of parametrization. 

The key limitation of the work includes the assumptions used to prove the results. Simplifying or removing Assumptions \ref{ass_transfer_error} and \ref{ass:sufficient_para} on the bias parameters is a valuable problem in the future work.

\bibliography{main.bib}
\appendix

\onecolumn
	\section{Proof of Helpful Lemmas}\label{sec_app_helpful}
	\subsection{Computation of the Gradient for Value Function}\label{sec:app_grad_computation}
	We compute the gradient for the reward value function and the gradient of constraint value function can be achieved similarly. Define $\tau=(s_0,a_1,s_1,a_1,s_2,a_2\cdots)$ as a trajectory, whose distribution induced by policy $\pi_\theta$ is $p(\tau\vert \theta)$  can be expressed as
	\begin{equation}\label{eq:traj_dis}
		p(\tau\vert \theta)=\rho(s_0)\prod_{t=0}^{\infty}\pi_\theta(a_t\vert s_t)P(s_{t+1}\vert s_t,a_t)
	\end{equation}
	Define $R(\tau)=\sum_{t=0}^{\infty}\gamma^tr(s_t,a_t)$ as the cumulative reward following the trajectory $\tau$. Then, the expected return $J_r^\pi(\theta)$ can  be expressed as
	\begin{equation*}
		J_r(\pi_\theta)=\mathbf{E}_{\tau\sim p(\tau\vert\theta)}[R(\tau)]
	\end{equation*}
	whose gradient is calculated as
	\begin{equation}\label{eq:grad_J}
		\nabla_\theta J_r(\pi_\theta)=\int_\tau R(\tau)p(\tau\vert \theta)d\tau =\int_{\tau}R(\tau)\frac{\nabla_\theta p(\tau\vert \theta)}{p(\tau\vert \theta)}p(\tau\vert \theta)d\tau=\mathbf{E}_{\tau\sim p}\big[\nabla_\theta \log p(\tau\vert\theta)R(\tau)\big]
	\end{equation}
	Notice that $\nabla_\theta \log p(\tau\vert \theta)$ is independent of the transition dynamics because 
	\begin{equation*}
		\nabla_\theta \log p(\tau\vert \theta)=\nabla_\theta\bigg[\log \rho(s_0)+\sum_{t=0}^{\infty}\big[\log\pi_\theta(a_t\vert s_t)+\log P(s_{t+1}\vert s_t,a_t)\big]\bigg]=\sum_{t=0}^\infty\nabla_\theta \log\pi_\theta(a_t\vert s_t)
	\end{equation*}
	and thus the gradient in Eq. \eqref{eq:grad_J} can be written as
	\begin{equation}\label{eq:gradient_origin}
		\nabla_\theta J_r(\pi_\theta)=\mathbf{E}_{\tau\sim p(\tau\vert \theta)}\bigg[\bigg(\sum_{t=0}^\infty \nabla_\theta \log \pi_\theta(a_t\vert s_t)\bigg)\sum_{t=0}^{\infty}\gamma^tr(s_t,a_t)\bigg]
	\end{equation}
	Notice that removing the past reward from the return doesn't change the expected value \cite{Peters2008}. Thus, we can rewrite Eq. \eqref{eq:gradient_origin} as
	\begin{equation}
		\nabla_\theta J_r(\pi_\theta)=\mathbf{E}_{\tau\sim p(\tau\vert \theta)}\bigg[\sum_{t=0}^\infty \nabla_\theta \log(\pi_\theta(a_t\vert s_t))\bigg(\sum_{h=t}^\infty \gamma^hr(s_h,a_h)\bigg)\bigg]
	\end{equation}
	\subsection{Proof of Lemma \ref{lem_smooth}}
		\begin{proof}
	    From Eq. \eqref{eq:grad_J}, we have
	    \begin{equation}
	        \begin{aligned}
		    \nabla_\theta^2J_r^{\pi_\theta}&=\nabla_\theta\int_\tau\nabla_\theta \log p(\tau|\theta)R_m(\tau)p(\tau|\theta)d\tau\\
		    &=\int_\tau\nabla_\theta^2 \log p(\tau|\theta)R(\tau)p(\tau|\theta)d\tau+\int_\tau\nabla_\theta \log p(\tau|\theta)R(\tau)\nabla_\theta p(\tau|\theta)d\tau
	        \end{aligned}
        \end{equation}
        For the first term, we have
        \begin{equation}
            \begin{aligned}
            \int_\tau\nabla_\theta^2 \log p(\tau|\theta)R(\tau)p(\tau|\theta)d\tau&=\mathbf{E}[ \sum_{t=0}^\infty\nabla_\theta^2 \log\pi_\theta(a_t\vert s_t)\sum_{t=0}^{\infty}\gamma^tr(s_t,a_t)]\\
            &=\mathbf{E}[ \sum_{t=0}^\infty\nabla_\theta^2 \log\pi_\theta(a_t\vert s_t)\sum_{h=t}^{\infty}\gamma^hr(s_h,a_h)]\\
            &\leq \frac{M}{(1-\gamma)^2}
            \end{aligned}
        \end{equation}
        where the last step is by Assumption \ref{ass_score}. For the second term, we have  \begin{equation}
	        \begin{aligned}
		    \int_\tau\nabla_\theta \log p(\tau|\theta)R_m(\tau)\nabla_\theta p(\tau|\theta)d\tau&=\mathbf{E}_{\tau\sim p}\bigg[[\nabla_\theta \log p(\tau|\theta)]^2R_m\bigg]\\
		    &=\mathbf{E}_{\tau\sim p}\bigg[\bigg(\sum_{t=0}^{\infty}\nabla_\theta \log \pi_\theta(a_t|s_t)\bigg)^2\sum_{h=0}^{\infty}\gamma^h r_m(s_h,a_h)\bigg]\\
		    &=\mathbf{E}_{\tau\sim p}\bigg[\sum_{h=0}^{\infty}\bigg(\sum_{t=0}^{h}\nabla_\theta \log \pi_\theta(a_t|s_t)\bigg)^2\gamma^h r_m(s_h,a_h)\bigg]\\
		    &\leq \frac{2G^2}{(1-\gamma)^3}
	        \end{aligned}
        \end{equation}
        Combining the two results together, we have $\Vert \nabla_\theta^2J_r^{\pi_\theta}\Vert\leq \frac{M}{(1-\gamma)^2}+\frac{2G^2}{(1-\gamma)^3}$, which gives the desired result.
	\end{proof}
	
	\subsection{Proof of Lemma \ref{lem_grad_bound}}
	\begin{proof}
	    Because the reward function is between $[0,1]$, we have
	    \begin{equation}
	        \begin{aligned}
	        \Vert\sum_{t=0}^\infty \nabla_\theta \log(\pi_\theta(a_t\vert s_t))\bigg(\sum_{h=t}^\infty \gamma^hr(s_h,a_h)\bigg)\Vert&\leq \frac{1}{1-\gamma}\sum_{t=0}^\infty \gamma^t\Vert \nabla_\theta \log(\pi_\theta(a_t\vert s_t))\Vert\\
	        &\leq \frac{G}{(1-\gamma)^2}
	        \end{aligned}
	    \end{equation}
	    where the last step is by Assumption \ref{ass_score}. Thus, $\Vert\nabla_\theta J_r(\theta)\Vert\leq \frac{G}{(1-\gamma)^2}$
	\end{proof}

\subsection{Proof of NPG update direction}
\begin{lemma}\label{lem:NPG_direction}
	    The minimizer of $L_{d_\rho^{\pi^\theta},\pi^\theta}(\omega,\theta,\bblambda)$ is the exact NPG update direction. Formally,
	    \begin{equation}
	        \omega_*^{\theta,\bblambda}=\argmin_{\omega}L_{d_\rho^{\pi^\theta},\pi^\theta}(\omega,\theta,\bblambda)=F_\rho(\theta)^{\dagger}\nabla_\theta J_{L,\bblambda}(\pi_\theta)
	    \end{equation}
	\end{lemma}
	\begin{proof}
	    Taking gradient of $L_{d_\rho^{\pi^\theta},\pi^\theta}(\omega,\theta,\bblambda)$ w.r.t. $\omega$, substituting $\omega=\omega_*$ and setting the equation equal to $0$, we have
	    \begin{equation}
		    \begin{aligned}
		    &\mathbf{E}_{s\sim d_\rho^{\pi_\theta}}\mathbf{E}_{a\sim\pi_\theta(\cdot\vert s)}\bigg[[\nabla_\theta\log\pi_{\theta}(a\vert s)\cdot(1-\gamma) \omega_*^{\theta,\bblambda}-A_{L,\bblambda}^{\pi_\theta}(s,a)][\nabla_\theta\log\pi_\theta(a\vert s)]\bigg]=0\\
		    \Rightarrow&\mathbf{E}_{s\sim d_\rho^{\pi_\theta}}\mathbf{E}_{a\sim\pi_\theta}[\nabla_\theta\log\pi_\theta(a\vert s)\nabla_\theta\log\pi_\theta(a\vert s)^T] \omega_*^{\theta,\bblambda}=\mathbf{E}_{s\sim d_\rho^{\pi_\theta}}\mathbf{E}_{a\sim\pi_\theta}[\frac{1}{1-\gamma}\nabla_\theta\log\pi_\theta(a\vert s)A_{L,\bblambda}^{\pi_\theta}(s,a)]\\
		    \Rightarrow&\mathbf{E}_{s\sim d_\rho^{\pi_\theta}}\mathbf{E}_{a\sim\pi_\theta}[\nabla_\theta\log\pi_\theta(a\vert s)\nabla_\theta\log\pi_\theta(a\vert s)^T] \omega_*^{\theta,\bblambda}=\mathbf{E}_{s\sim d_\rho^{\pi_\theta}}\mathbf{E}_{a\sim\pi_\theta}[\frac{1}{1-\gamma}\nabla_\theta\log\pi_\theta(a\vert s)Q_{L,\bblambda}^{\pi_\theta}(s,a)]\\
		    \Rightarrow&F_\rho(\theta)^{\dagger} \omega_*^{\theta,\bblambda}=\mathbf{E}_{s\sim d_\rho^{\pi_\theta}}\mathbf{E}_{a\sim\pi_\theta}[\frac{1}{1-\gamma}\nabla_\theta\log\pi_\theta(a\vert s)Q_{r}^{\pi_\theta}(s,a)]+\sum_{i\in[I]}\lambda^i\mathbf{E}_{s\sim d_\rho^{\pi_\theta}}\mathbf{E}_{a\sim\pi_\theta}[\frac{1}{1-\gamma}\nabla_\theta\log\pi_\theta(a\vert s)Q_{g^i}^{\pi_\theta}(s,a)]\\
	    	\overset{(a)}\Rightarrow&F_\rho(\theta)^\dagger \omega_*^{\theta,\bblambda}=\nabla_\theta J_r(\pi_\theta)+\sum_{i\in[I]}\lambda^i \nabla_\theta J_{g^i}(\pi_\theta)\Rightarrow 			F_\rho(\theta) \omega_*^{\theta,\bblambda}=\nabla_\theta J_{L,\bblambda}(\pi_\theta)\Rightarrow \omega_*^{\theta,\bblambda}=F_\rho(\theta)^{\dagger}\nabla_\theta J_{L,\bblambda}(\pi_\theta)
	       \end{aligned}
    	\end{equation}
    	where step (a) holds by the Policy Gradient Theorem.
	\end{proof}
    \begin{lemma}[Strong duality]\citep[Lemma 1]{Ding2020}
    \label{lem.duality}
    For convenience, we rewrite the conservative problem
\begin{equation}\label{eq:rewrite_unparameterized}
    \begin{aligned}
        \max_{\pi\in\Pi} ~& J_r^{\pi} \\
        \text{s.t.} ~& J_g^{\pi}\geq \kappa
    \end{aligned}
\end{equation} 

Define $\pi^*$ as the optimal solution to the above problem. Define the associated dual function as
\begin{equation}
    J_D^{\lambda}\triangleq\max_{\pi\in\Pi} J_r^{\pi}+\lambda (J_g^{\pi}-\kappa)
\end{equation}
and denote $\lambda^*=\arg\min_{\lambda\geq 0} J_D^{\lambda}$. We have the following strong duality property for the unparameterized problem.
    \begin{equation}\label{eq:duality}
        J_r^{\pi^*} = J_D^{\lambda^*} 
    \end{equation}	
\end{lemma}

Although the strong duality holds for the unparameterized problem, the same is not true for parameterized class $\{\pi_\theta|\theta\in \Theta\}$. To formalize this statement, define the dual function associated with the parameterized problem as follows.
\begin{equation}
    J_{D,\Theta}^{\lambda}\triangleq\max_{\theta\in \Theta}  J_r(\theta)+\lambda(J_g(\theta)-\kappa)
\end{equation}
and denote $\lambda_\Theta^*=\arg\min_{\lambda\geq 0} J_{D,\Theta}^{\lambda}$. The lack of strong duality states that, in general, $J_{D, \Theta}^{\lambda_{\Theta}^*}\neq J_r(\theta^*)$ where $\theta^*$ is a solution of the parameterized constrained optimization. However, the parameter $\lambda_\Theta^*$, as we demonstrate below, must obey some restrictions.
\begin{lemma}
	\label{lem.boundness}
	Under  Assumption \ref{ass_slater}, the optimal dual variable for the parameterized problem is bounded as
    \begin{equation}
        0 \leq \lambda_\Theta^* \leq \frac{J_r^{\pi^*}-J_r(\bar{\theta})}{\varphi}\leq \dfrac{1}{(1-\gamma)\varphi}
    \end{equation}
\end{lemma}

\begin{proof}
    The proof follows the approach in \citep[Lemma 1]{Ding2020}, but is revised to the general parameterization setup.	Let $\Lambda_a\triangleq\{ \lambda\geq 0\,\vert\, J_{D,\Theta}^\lambda \leq a \}$ be a sublevel set of the dual function for $a\in\mathbb{R}$. If $\Lambda_a$ is non-empty, then for any $\lambda \in\Lambda_a$, 
	\begin{equation}
	    a\geq J_{D,\Theta}^\lambda\geq J_r(\bar{\theta})+\lambda (J_g(\bar{\theta})-\kappa)\geq J_r(\bar{\theta})+\lambda \varphi
	\end{equation}
    where $\bar{\theta}$ is a Slater point in Assumption \ref{ass_slater}. Thus, $\lambda \leq (a -J_r(\bar{\theta}))/\varphi$.	If we take $a= J_{D,\Theta}^{\lambda_\Theta^*}\leq J_{D,\Theta}^{\lambda^*} \leq J_D^{\lambda^*}=J_r^{\pi^*}$, then we have $\lambda_\Theta^*\in \Lambda_a$, which proves the Lemma. The last inequality holds since $J_r^{\pi}\in [0,1]$ for any policy, $\pi$.
\end{proof}

    Since the above inequality holds for arbitrary $\Theta$, we also have, $0\leq \lambda^*\leq \frac{1}{(1-\gamma)\varphi}$. Define $v(\tau)\triangleq\max_{\pi\in\Pi}\{J_r^\pi|J_g^\pi\geq \tau+\kappa\}$. Using the strong duality property of the unparameterized problem \eqref{eq:rewrite_unparameterized}, we establish the following property of the function, $v(\cdot)$.

\begin{lemma}
    \label{lem:bridge}
    Assume that the Assumption \ref{ass_slater} holds, we have for any $\tau\in\mathbb{R}$,
    \begin{equation}
        v(0)-\tau\lambda^* \geq	v(\tau)
    \end{equation}
\end{lemma}

\begin{proof}
    By the definition of $v(\tau)$, we have $v(0) = J_r^{\pi^*}$. With a slight abuse of notation, denote $J_{\mathrm{L}}(\pi,\lambda)=J_r^{\pi}+\lambda (J_g^{\pi}-\kappa)$. By the strong duality stated in Lemma \ref{lem.duality}, we have the following for any $\pi\in\Pi$.
    \begin{equation}
        J_{\mathrm{L}}(\pi,\lambda^*)\leq \max_{\pi\in\Pi} J_{\mathrm{L}}(\pi,\lambda^*)\overset{Def}=J_D^{\lambda^*}\overset{\eqref{eq:duality}}=J_r^{\pi^*}=v(0)
    \end{equation}
    Thus, for any $\pi\in\{ \pi\in\Pi \,\vert\,J_g^{\pi} \geq \tau+\kappa \}$,
    \begin{equation}
	\begin{aligned}
            v(0)-\tau\lambda^*&\geq J_L(\pi,\lambda^*)-\tau\lambda^*\\
            &=J_r^{\pi}+\lambda^*(J_g^\pi-\tau-\kappa) \geq J_r^{\pi}
	\end{aligned}
    \end{equation}
    Maximizing the right-hand side of this inequality over $\{ \pi\in\Pi \vert J_{c}^{\pi}\geq \tau+\kappa \}$ yields
    \begin{equation}
        \label{eq.opt1}
	v(0)- \tau\lambda^* \geq v(\tau)
    \end{equation}
    This completes the proof of the lemma.
\end{proof}

We note that a similar result was shown in \citep[Lemma 15]{bai2023provably}. However, the setup of the stated paper is different from that of ours. Specifically, \cite{bai2023provably} considers a tabular setup with peak constraints. Note that Lemma \ref{lem:bridge} has no direct connection with the parameterized setup since its proof uses strong duality and the function, $v(\cdot)$, is defined via a constrained optimization over the entire policy set, $\Pi$, rather than the parameterized policy set. Interestingly, however, the relationship between $v(\tau)$ and $v(0)$ leads to the lemma stated below which turns out to be pivotal in establishing regret and constraint violation bounds in the parameterized setup.

\begin{lemma}\label{lem.constraint}
    Let Assumption \ref{ass_slater} hold. For any constant $C\geq 2\lambda^*$, if there exists a $\pi\in\Pi$ and $\zeta>0$ such that $J_r^{\pi^*}-J_r^{\pi}+C[\kappa-J_g^{\pi}]_{+}\leq \zeta$, then 
    \begin{equation}
        \left[\kappa-J_g^{\pi}\right]_+\leq 2\zeta/C
    \end{equation}
    where $[x]_+=\max(x,0)$.
\end{lemma}

\begin{proof}
    Note that the Lemma is trivially true if $\kappa-J^{\pi}_g\leq 0$. Assume, $\kappa-J^{\pi}_g>0$. Let $\tau = -\left[\kappa-J_g^{\pi}\right]_+ = J^{\pi}_g$. Using the definition of $v(\tau)$, one can write,
	\begin{equation}\label{eq.opt2}
		J_r^{\pi}\leq v(\tau)
	\end{equation}
    Combining Eq. \eqref{eq.opt1} and \eqref{eq.opt2}, we obtain the following.
    \begin{equation}
        J_r^{\pi}-J_r^{\pi^*}\leq v(\tau)-v(0)\leq -\tau\lambda^*
    \end{equation}
	The condition in the Lemma leads to,
	\begin{equation}
	    (C - \lambda^*) |\tau| = 
		{\tau} \lambda^*+C |\tau|
        \leq 
		J_r^{\pi^*}-J_r^{\pi}+C |\tau|\leq \zeta
	\end{equation}
    Finally, we have,
    \begin{equation}
        |\tau|\leq \frac{\zeta}{C-\lambda^*}\leq\frac{2\zeta}{C}
    \end{equation}
	which completes the proof.
\end{proof}
\section{Convergence of SGD Procedure}
    In order to show that the SGD procedure converges to the exact NPG update direction, we need the following lemma.
   \begin{lemma}\label{lem:bound_omega}
	    For any NPG update iteration $k$, the exact NPG update direction $\omega_*^k$ is bounded. Formally,
	    \begin{equation}
	        \Vert\omega_*^k\Vert_2\leq \frac{G(1+I\Lambda)}{\mu_F(1-\gamma)^2}
	    \end{equation}
	\end{lemma}
	\begin{proof}
	    Recall that $\omega_*^k$ is the NPG update direction for iteration $k$. Thus, by definition of NPG algorithm, we have
	    \begin{equation}
	        \begin{aligned}
	        \Vert\omega_*^k\Vert_2&=\Vert F^{-1}(\theta^k)\nabla_\theta J_L(\theta^k,\bblambda^k)\Vert_2\overset{(a)}\leq \Vert F^{-1}(\theta^k)\Vert_2 \Vert\nabla_\theta J_L(\theta^k,\bblambda^k)\Vert_2\\
	        &\overset{(b)}\leq \frac{\Vert\nabla_\theta J_r(\theta^k)\Vert_2+\sum_{i\in I}\lambda_i^k \Vert\nabla_\theta J_{g^i}(\theta^k)\Vert_2}{\mu_F}\\
	        &\overset{(c)}\leq \frac{G(1+\sum_{i\in I}\lambda_i^k)}{\mu_F(1-\gamma)^2}\leq \frac{G(1+I\Lambda)}{\mu_F(1-\gamma)^2}
	        \end{aligned}
	   \end{equation}
	   where the step (a) is true due to the property of matrix norm. Step (b) holds by  Assumption \ref{ass_pd} and the triangle inequality. Step (c) holds by Lemma \ref{lem_grad_bound}. 
	\end{proof}
\begin{lemma}\label{lem:bound_omega_diff}
    In the SGD procedure, setting the learning rate $\alpha=\frac{1}{4G^2}$, for any NPG update iteration $k$, we have
    \begin{equation}
        \mathbf{E}[\Vert \omega^k-\omega_*^k\Vert_2^2]\leq \frac{4}{N\mu_F}\bigg[2[\frac{G^2(1+I\Lambda)}{\mu_F(1-\gamma)^2}+\frac{2}{(1-\gamma)^2}]\sqrt{d}+\frac{G^2(1+I\Lambda)}{\mu_F(1-\gamma)^2}\bigg]^2
    \end{equation}
    If we take the number of samples as $N=\ccalO(\frac{I^2\Lambda^2}{(1-\gamma)^4\epsilon})$, then
	\begin{equation}
        \mathbf{E}[\Vert \omega^k-\omega_*^k\Vert_2^2]\leq\epsilon
    \end{equation}
\end{lemma}
\begin{proof}
	From the definition of the NPG update direction, we have
	\begin{equation}
		\omega_*^k=F_\rho(\theta_k)^{\dagger}\nabla_\theta J_L(\pi_\theta^k,\bblambda^k)
	\end{equation}
	which is also the minimizer of the compatible function approximation error
	\begin{equation}\label{eq:compatible_function_error}
		\omega_*^k=\argmin_{\omega}L_{d_\rho^{\pi},\pi}(\omega,\theta^k,\bblambda^k)=\mathbf{E}_{s\sim d_\rho^{\pi}}\mathbf{E}_{a\sim\pi(\cdot\vert s)}\bigg[\bigg(\nabla_\theta\log\pi_\theta^k(a\vert s)\cdot(1-\gamma)\omega-A_{L,\bblambda^k}^{\pi_\theta^k}(s,a)\bigg)^2\bigg]
	\end{equation}
	The gradient of the compatible function approximation error can be obtained as
	\begin{equation}
		\nabla_\omega L_{d_\rho^{\pi},\pi}(\omega,\theta^k,\bblambda^k) = 2(1-\gamma)\mathbf{E}_{s\sim d_\rho^{\pi}}\mathbf{E}_{a\sim\pi(\cdot\vert s)}\bigg[\nabla_\theta\log\pi_\theta^k(a\vert s)\cdot(1-\gamma)\omega-A_{L,\bblambda^k}^{\pi_\theta^k}(s,a)\bigg]\nabla_\theta\log\pi_\theta^k(a\vert s)
	\end{equation}
	In the SGD procedure, the stochastic version of gradient can be written as
	\begin{equation}
		\hat{\nabla}_\omega L_{d_\rho^{\pi},\pi}(\omega,\theta^k,\bblambda^k)= 2(1-\gamma)\bigg[\nabla_\theta\log\pi_\theta^k(a\vert s)\cdot(1-\gamma)\omega-\hat{A}_{L,\bblambda^k}^{\pi_\theta^k}(s,a)\bigg]\nabla_\theta\log\pi_\theta^k(a\vert s)
	\end{equation}
	where $\hat{A}_{L,\bblambda^k}^{\pi_\theta^k}$ is an unbiased estimator for $A_{L,\bblambda^k}^{\pi_\theta^k}$ \cite{Ding2020}. Setting the SGD learning rate $\alpha=\frac{1}{4G^2}$, \cite{Non-strongly-convex}[Theorem 1] gives
	\begin{equation}
	    \begin{aligned}
	    \mathbf{E}[L_{d_\rho^{\pi},\pi}(\omega^k,\theta^k,\bblambda^k)-L_{d_\rho^{\pi},\pi}(\omega_*^k,\theta^k,\bblambda^k)]
	    &\leq \frac{2}{N}[\sigma\sqrt{d}+G\Vert\omega_0^k-\omega_k^*\Vert]^2\\
	    &\leq \frac{2}{N}[\sigma\sqrt{d}+\frac{G^2(1+I\Lambda)}{\mu_F(1-\gamma)^2}]^2
	    \end{aligned}
	\end{equation}
	The last step holds by Lemma \ref{lem:bound_omega}. Here, $d$ is the dimension for parameterization $\theta$ and $\sigma$ needs to satisfy
	\begin{equation}
		\mathbf{E}	\bigg[\hat{\nabla}_\omega L_{d_\rho^{\pi},\pi}(\omega,\theta^k,\bblambda^k)	\hat{\nabla}_\omega^T L_{d_\rho^{\pi},\pi}(\omega,\theta^k,\bblambda^k)\bigg]\preceq(1-\gamma)^2\sigma^2 \nabla_\theta\log\pi_\theta^k(a\vert s)\nabla_\theta\log\pi_\theta^k(a\vert s)^T
	\end{equation}
	One feasible choice of $\sigma$ is $\sigma=2[\frac{G^2(1+I\Lambda)}{\mu_F(1-\gamma)^2}+\frac{2}{(1-\gamma)^2}]$ and notice that $L_{d_\rho^{\pi},\pi}(\omega^k,\theta^k,\bblambda^k)$ is $\mu_F$-strongly convex with respect to $\omega^k$, thus
	\begin{equation}
	    \begin{aligned}
	    L_{d_\rho^{\pi},\pi}(\omega^k,\theta^k,\bblambda^k)
	    &\geq L_{d_\rho^{\pi},\pi}(\omega_*^k,\theta^k,\bblambda^k)+\nabla_\theta L_{d_\rho^{\pi},\pi}(\omega_*^k,\theta^k,\bblambda^k)(\omega^k-\omega_*^k)+\frac{\mu_F}{2}\Vert\omega^k-\omega_*^k\Vert_2^2\\
	    &=L_{d_\rho^{\pi},\pi}(\omega_*^k,\theta^k,\bblambda^k)+\frac{\mu_F}{2}\Vert\omega^k-\omega_*^k\Vert_2^2
	    \end{aligned}
	\end{equation}
	where the second step holds because $\omega_*^k$ is the minimizer of $L_{d_\rho^{\pi},\pi}(\omega^k,\theta^k,\bblambda^k)$ and thus $\nabla_\theta L_{d_\rho^{\pi},\pi}(\omega_*^k,\theta^k,\bblambda^k)=0$. Rearranging items and taking expectation on both side, we have
	\begin{equation}
	    \begin{aligned}
	    \mathbf{E}[\Vert \omega^k-\omega_*^k\Vert_2^2]&\leq \frac{2}{\mu_F}\mathbf{E}[	L_{d_\rho^{\pi},\pi}(\omega^k,\theta^k,\bblambda^k)-	L_{d_\rho^{\pi},\pi}(\omega_*^k,\theta^k,\bblambda^k)]\\
	    &\leq \frac{4}{N\mu_F}\bigg[2[\frac{G^2(1+I\Lambda)}{\mu_F(1-\gamma)^2}+\frac{2}{(1-\gamma)^2}]\sqrt{d}+\frac{G^2(1+I\Lambda)}{\mu_F(1-\gamma)^2}\bigg]^2
	    \end{aligned}
	\end{equation}
	If we take the number of samples $N$ as
	\begin{equation}
	    N=\frac{4\bigg[2[\frac{G^2(1+I\Lambda)}{\mu_F(1-\gamma)^2}+\frac{2}{(1-\gamma)^2}]\sqrt{d}+\frac{G^2(1+I\Lambda)}{\mu_F(1-\gamma)^2}\bigg]^2}{\mu_F\epsilon}=\ccalO\bigg(\frac{I^2\Lambda^2}{(1-\gamma)^4\epsilon}\bigg)
	\end{equation}
	we have
	\begin{equation}
	    \mathbf{E}[\Vert \omega^k-\omega_*^k\Vert_2^2]\leq \epsilon
	\end{equation}
	\end{proof}
 \section{First order stationary}\label{sec:first_order}
\begin{proof}
    By Lemma \ref{lem_smooth}, both $J_r$ and $J_{g^i},i\in[I]$ are $L_J$-smooth w.r.t $\theta$. Thus, for any fixed $\bblambda^k$, $J_L(\theta,\bblambda^k)$ is still $L_J$-smooth w.r.t $\theta$, which gives,
	\begin{equation}
	    \begin{aligned}
	        J_L(\theta^{k+1},\bblambda^k)&\geq J_L(\theta^{k+1},\bblambda^k)+\left<\nabla_\theta J_L(\theta^k,\bblambda^k),\theta^{k+1}-\theta^k\right>-\frac{L_J}{2}\Vert\theta^{k+1}-\theta^k\Vert_2^2\\
	        &=J_L(\theta^{k+1},\bblambda^k)+\left<\nabla_\theta J_L(\theta^k,\bblambda^k),\theta_*^{k+1}-\theta^k\right>+\left<\nabla_\theta J_L(\theta^k,\bblambda^k),\theta^{k+1}-\theta_*^{k+1}\right>-\frac{L_J}{2}\Vert\theta^{k+1}-\theta^k\Vert_2^2\\
	        &\overset{(a)}=J_L(\theta^{k+1},\bblambda^k)+
	        \eta_1\left<\nabla_\theta J_L(\theta^k,\bblambda^k),F^{-1}(\theta^k)\nabla_\theta J_L(\theta^{k},\bblambda^k)\right>+\left<\nabla_\theta J_L(\theta^k,\bblambda^k),\theta^{k+1}-\theta_*^{k+1}\right>\\
	        &\quad-\frac{L_J}{2}\Vert\theta^{k+1}-\theta^k\Vert_2^2\\
	        &\overset{(b)}\geq J_L(\theta^{k+1},\bblambda^k)+
	        \frac{\eta_1}{G^2}\Vert \nabla_\theta J_L(\theta^k,\bblambda^k)\Vert_2^2+\left<\nabla_\theta J_L(\theta^k,\bblambda^k),\theta^{k+1}-\theta_*^{k+1}\right>-\frac{L_J}{2}\Vert\theta^{k+1}-\theta^k\Vert_2^2\\
	        &\overset{(c)}\geq J_L(\theta^{k+1},\bblambda^k)+
	        \frac{\eta_1}{2G^2}\Vert \nabla_\theta J_L(\theta^k,\bblambda^k)\Vert_2^2-\frac{G^2}{2\eta_1}\Vert \theta^{k+1}-\theta_*^{k+1}\Vert^2-\frac{L_J}{2}\Vert\theta^{k+1}-\theta^k\Vert_2^2\\
	        &\geq J_L(\theta^{k+1},\bblambda^k)+
	        \frac{\eta_1}{2G^2}\Vert \nabla_\theta J_L(\theta^k,\bblambda^k)\Vert_2^2-\bigg(\frac{G^2}{2\eta_1}+L_J\bigg)\Vert \theta^{k+1}-\theta_*^{k+1}\Vert^2-L_J\Vert\theta_*^{k+1}-\theta^k\Vert_2^2\\
	        &\overset{(d)}\geq J_L(\theta^{k+1},\bblambda^k)+
	        \bigg(\frac{\eta_1}{2G^2}-\frac{L_J\eta_1^2}{\mu_F^2}\bigg)\Vert \nabla_\theta J_L(\theta^k,\bblambda^k)\Vert_2^2-\bigg(\frac{G^2}{2\eta_1}+L_J\bigg)\Vert \theta^{k+1}-\theta_*^{k+1}\Vert^2\\
	    \end{aligned}
	\end{equation}
	where steps (a) and (d) hold because $\theta_*^{k+1}=\theta^k+\eta F^{-1}(\theta^k)\nabla J_L(\theta^k,\bblambda^k)$ in Algorithm \ref{alg:spdgd}. Steps (b) and (c) hold by Assumption \ref{ass_score} and Young's Inequality \cite{Young}, respectively. Adding $J_L(\theta^{k+1},\bblambda^{k+1})$ on both sides, we have
	\begin{equation}\label{eq:bound_first_order2}
	    \begin{aligned}
	    J_L(\theta^{k+1},\bblambda^{k+1})
	    &\geq J_L(\theta^{k+1},\bblambda^k)+J_L(\theta^{k+1},\bblambda^{k+1})-J_L(\theta^{k+1},\bblambda^k)+\bigg(\frac{\eta_1}{2G^2}-\frac{L_J\eta_1^2}{\mu_F^2}\bigg)\Vert \nabla_\theta J_L(\theta^k,\bblambda^k)\Vert_2^2)\\
	    &-\bigg(\frac{G^2}{2\eta_1}+L_J\bigg)\Vert \theta^{k+1}-\theta_*^{k+1}\Vert^2\\
	    &\overset{(a)}=J_L(\theta^{k+1},\bblambda^k)+\sum_{i\in[I]}(\lambda_i^{k+1}-\lambda_i^{k})J_{g^i}(\theta^{k+1})+\bigg(\frac{\eta_1}{2G^2}-\frac{L_J\eta_1^2}{\mu_F^2}\bigg)\Vert \nabla_\theta J_L(\theta^k,\bblambda^k)\Vert_2^2\\
	    &-\bigg(\frac{G^2}{2\eta_1}+L_J\bigg)\Vert \theta^{k+1}-\theta_*^{k+1}\Vert^2\\
	    &\overset{(b)}\geq J_L(\theta^{k+1},\bblambda^k)-\frac{1}{1-\gamma}\sum_{i\in[I]}|\lambda_i^{k+1}-\lambda_i^{k}|+\bigg(\frac{\eta_1}{2G^2}-\frac{L_J\eta_1^2}{\mu_F^2}\bigg)\Vert \nabla_\theta J_L(\theta^k,\bblambda^k)\Vert_2^2\\
	    &-\bigg(\frac{G^2}{2\eta_1}+L_J\bigg)\Vert \theta^{k+1}-\theta_*^{k+1}\Vert^2\\
	    &\overset{(c)}\geq J_L(\theta^{k+1},\bblambda^k)-\frac{1}{1-\gamma}\sum_{i\in[I]}|\lambda_i^{k+1}-\lambda_i^{k}|+\bigg(\frac{\eta_1}{2G^2}-\frac{L_J\eta_1^2}{\mu_F^2}\bigg)\Vert \nabla_\theta J_L(\theta^k,\bblambda^k)\Vert_2^2\\
	    &-\eta_1^2\bigg(\frac{G^2}{2\eta_1}+L_J\bigg)\Vert \omega^k-\omega_*^k\Vert^2\\
	    \end{aligned}
	\end{equation}
	where step (a) holds by definition $J_L(\theta,\bblambda)=J_r(\theta)+\sum_{i\in[I]}\lambda^i J_{g^i}(\theta)$ and step (b) is true due to $|J_{g^i}(\theta)|\leq \frac{1}{1-\gamma}$. The last step (c) is true because $\theta_*^{k+1}=\theta^k+\eta F^{-1}(\theta^k)\nabla J_L(\theta^k,\bblambda^k)$ and thus
	\begin{equation}
	    \begin{aligned}
	    \theta^{k+1}-\theta_*^{k+1}&=\theta^{k+1}-\theta^k-\eta F^{-1}(\theta^k)\nabla J_L(\theta^k,\bblambda^k)\\
	    &=\eta_1\bigg(\omega^k-F^{-1}(\theta^k)\nabla J_L(\theta^k,\bblambda^k)\bigg)\\
	    &=\eta_1(\omega^k-\omega_*^k)
	    \end{aligned}
	\end{equation}
	Taking the expectation on the both sides of Eq. \eqref{eq:bound_first_order2}, we have
	\begin{equation}\label{eq:bound_first_order4}
	    \begin{aligned}
	    \mathbf{E}[J_L(\theta^{k+1},\bblambda^{k+1})]
	    &\geq \mathbf{E}[J_L(\theta^{k+1},\bblambda^k)]-\frac{1}{1-\gamma}\sum_{i\in[I]}\mathbf{E}|\lambda_i^{k+1}-\lambda_i^{k}|+\bigg(\frac{\eta_1}{2G^2}-\frac{L_J\eta_1^2}{\mu_F^2}\bigg)\mathbf{E}\Vert \nabla_\theta J_L(\theta^k,\bblambda^k)\Vert_2^2\\
	    &-\eta_1^2\bigg(\frac{G^2}{2\eta_1}+L_J\bigg)\mathbf{E}[\Vert\omega^k-\omega_*^k\Vert_2^2]\\
	    &\geq \mathbf{E}[J_L(\theta^{k+1},\bblambda^k)]-\frac{1}{1-\gamma}\sum_{i\in[I]}\mathbf{E}|\lambda_i^{k+1}-\lambda_i^{k}|+\bigg(\frac{\eta_1}{2G^2}-\frac{L_J\eta_1^2}{\mu_F^2}\bigg)\mathbf{E}\Vert \nabla_\theta J_L(\theta^k,\bblambda^k)\Vert_2^2\\
	    &-\frac{4\eta_1^2}{N\mu_F}\bigg(\frac{G^2}{2\eta_1}+L_J\bigg)\bigg[2[\frac{G^2(1+I\Lambda)}{\mu_F(1-\gamma)^2}+\frac{2}{(1-\gamma)^2}]\sqrt{d}+\frac{G^2(1+I\Lambda)}{\mu_F(1-\gamma)^2}\bigg]^2
	    \end{aligned}
	\end{equation}
	Summing $k$ from 0 to $K-1$ and dividing by $K$ on both sides, recall $\lambda_i^{k+1}=P_{[0,\frac{2}{(1-\gamma)\varphi}]}[\lambda_i^k-\eta_2J_g(\theta^k)]$, by the non-expansive of projection, we have
	\begin{equation}
	    \begin{aligned}
	    \frac{\mathbf{E}[J_L(\theta^{K},\bblambda^{K})]-J_L(\theta^{0},\bblambda^{0})}{K}&\geq -\frac{\eta_2 I}{(1-\gamma)^2}+\bigg(\frac{\eta_1}{2G^2}-\frac{L_J\eta_1^2}{\mu_F^2}\bigg)\frac{1}{K}\sum_{k=0}^{K-1}\mathbf{E}\Vert \nabla_\theta J_L(\theta^k,\bblambda^k)\Vert_2^2\\
	    &-\frac{4\eta_1^2}{N\mu_F}\bigg(\frac{G^2}{2\eta_1}+L_J\bigg)\bigg[2[\frac{G^2(1+I\Lambda)}{\mu_F(1-\gamma)^2}+\frac{2}{(1-\gamma)^2}]\sqrt{d}+\frac{G^2(1+I\Lambda)}{\mu_F(1-\gamma)^2}\bigg]^2
	    \end{aligned}
	\end{equation}
	Recall $\Lambda=\frac{2}{(1-\gamma)\varphi}$ and noticing that $|J_L(\theta,\bblambda)|\leq \frac{1+I\Lambda}{1-\gamma}$ for any $\theta$ and $\bblambda$, we have
	\begin{equation}
	    \begin{aligned}
	    \frac{2+2I\Lambda}{K(1-\gamma)}+\frac{\eta_2 I}{(1-\gamma)^2}
	    &\geq \bigg(\frac{\eta_1}{2G^2}-\frac{L_J\eta_1^2}{\mu_F^2}\bigg)\frac{1}{K}\sum_{k=0}^{K-1}\mathbf{E}\Vert \nabla_\theta J_L(\theta^k,\bblambda^k)\Vert_2^2\\
	    &-\frac{4\eta_1^2}{N\mu_F}\bigg(\frac{G^2}{2\eta_1}+L_J\bigg)\bigg[2[\frac{G^2(1+I\Lambda)}{\mu_F(1-\gamma)^2}+\frac{2}{(1-\gamma)^2}]\sqrt{d}+\frac{G^2(1+I\Lambda)}{\mu_F(1-\gamma)^2}\bigg]^2
	    \end{aligned}
	\end{equation}
	Rearranging items and letting $\eta_1=\frac{\mu_F^2}{4G^2L_J}$, we have
	\begin{equation}
	    \begin{aligned}
	    \frac{1}{K} \sum_{k=0}^{K-1}\mathbf{E}\Vert \nabla_\theta J_L(\theta^k,\bblambda^k)\Vert_2^2
	    &\leq \frac{\frac{2+2I\Lambda}{K(1-\gamma)}+\frac{\eta_2 I}{(1-\gamma)^2}+\frac{4\eta_1^2}{N\mu_F}\bigg(\frac{G^2}{2\eta_1}+L_J\bigg)\bigg[2[\frac{G^2(1+I\Lambda)}{\mu_F(1-\gamma)^2}+\frac{2}{(1-\gamma)^2}]\sqrt{d}+\frac{G^2(1+I\Lambda)}{\mu_F(1-\gamma)^2}\bigg]^2}{\frac{\eta_1}{2G^2}-\frac{L_J\eta_1^2}{\mu_F^2}}\\
	    &\leq \frac{\frac{2+2I\Lambda}{K(1-\gamma)}+\frac{\eta_2 I}{(1-\gamma)^2}+\frac{\mu_F(\mu_F^2+2G^4)}{2G^4L_J N}\bigg[2[\frac{G^2(1+I\Lambda)}{\mu_F(1-\gamma)^2}+\frac{2}{(1-\gamma)^2}]\sqrt{d}+\frac{G^2(1+I\Lambda)}{\mu_F(1-\gamma)^2}\bigg]^2}{\frac{\mu_F^2}{16G^4L_J}}\\
	    &=\frac{8(\mu_F^2+2G^4)}{\mu_F N}\bigg[2[\frac{G^2(1+I\Lambda)}{\mu_F(1-\gamma)^2}+\frac{2}{(1-\gamma)^2}]\sqrt{d}+\frac{G^2(1+I\Lambda)}{\mu_F(1-\gamma)^2}\bigg]^2\\
        &+\frac{16G^4L_J}{\mu_F^2(1-\gamma)}[\frac{2(1+I\Lambda)}{K}+\frac{\eta_2I}{(1-\gamma)}]
	    \end{aligned}
	\end{equation}
	Take $\eta_2=\frac{1}{\sqrt{K}}$ and 
	\begin{equation}
	    \begin{aligned}
	        K&=\frac{256I^2G^8L_J^2}{\mu_F^4(1-\gamma)^4\epsilon^2}=\ccalO\bigg(\frac{I^2}{(1-\gamma)^4\epsilon^2}\bigg)\\
	        N&=\frac{16(\mu_F^2+2G^4)}{\mu_F\epsilon}\bigg[2[\frac{G^2(1+I\Lambda)}{\mu_F(1-\gamma)^2}+\frac{2}{(1-\gamma)^2}]\sqrt{d}+\frac{G^2(1+I\Lambda)}{\mu_F(1-\gamma)^2}\bigg]^2=\ccalO\bigg(\frac{I^2\Lambda^2}{(1-\gamma)^4\epsilon}\bigg)
	    \end{aligned}
	\end{equation}
	we have
	\begin{equation}
	    \frac{1}{K} \sum_{k=0}^{K-1}\mathbf{E}\Vert \nabla_\theta J_L(\theta^k,\bblambda^k)\Vert_2^2\leq \epsilon
	\end{equation}
	Thus, we overall need to sample
	\begin{equation}
	    \ccalO\bigg(\frac{I^2}{(1-\gamma)^4\epsilon}\bigg)\cdot \ccalO\bigg(\frac{I^2\Lambda^2}{(1-\gamma)^4\epsilon}\bigg)=\ccalO\bigg(\frac{I^4\Lambda^2}{(1-\gamma)^8\epsilon^3}\bigg)
	\end{equation}
	trajectories.
\end{proof}
\section{Proof of Lemma \ref{lem_bound_JL}}\label{sec_app_framwork}
\subsection{The General Framework of Global Convergence}
	\begin{lemma}\label{lem_framework}
    	Suppose a general primal-dual gradient ascent algorithm updates the parameter as
    	\begin{equation}
    	    \begin{aligned}
    	   	&\theta^{k+1}=\theta^k+\eta\omega^k\\
	    	&\lambda_i^{k+1}=\ccalP_{(0,\Lambda]}\bigg(\lambda_i^k-\eta_2 \big(J_{g^i}(\theta^k)-\kappa\big)\bigg) 
    	    \end{aligned}
    	\end{equation}
    	When Assumptions \ref{ass_score} and \ref{ass_transfer_error} hold, we have
        \begin{equation}
	        \frac{1}{K}\sum_{k=1}^{K}\mathbf{E}\bigg(J_L(\pi^*_{\theta,\kappa},\bblambda^k)-J_L(\pi_\theta^k,\bblambda^k)\bigg)\leq \frac{\sqrt{\epsilon_{bias}}}{1-\gamma}+\frac{M\eta_1}{2K}\sum_{k=0}^{K-1}\mathbf{E}\Vert\omega^k\Vert^2+\frac{\log(|\ccalA|)}{\eta_1 K}+\frac{G}{K}\sum_{k=1}^{K}\mathbf{E}\Vert(\omega^k-\omega_*^k)\Vert_2
	   \end{equation}
	    where $\omega_*^k:=\omega_*^{\theta^k}$ and is defined in Eq. \eqref{eq:NPG_direction}
    \end{lemma}
    The above Lemma extends the result in \cite{Yanli2020}[Proposition 4.5] to the constrained Markov decision Process with conservative constraints.
\begin{proof}
	Starting with the definition of KL divergence,
	\begin{equation}
	\begin{aligned}
	&\mathbf{E}_{s\sim d_\rho^{\pi^*_{\theta,\kappa}}}[KL(\pi^*_{\theta,\kappa}(\cdot\vert s)\Vert\pi_{\theta^k}(\cdot\vert s))-KL(\pi^*_{\theta,\kappa}(\cdot\vert s)\Vert\pi_{\theta^{k+1}}(\cdot\vert s))]\\
	=&\mathbf{E}_{s\sim d_\rho^{\pi^*_{\theta,\kappa}}}\mathbf{E}_{a\sim\pi^*_{\theta,\kappa}(\cdot\vert s)}\bigg[\log\frac{\pi_{\theta^{k+1}}(a\vert s)}{\pi_{\theta^k}(a\vert s)}\bigg]\\
	\overset{(a)}\geq&\mathbf{E}_{s\sim d_\rho^{\pi^*_{\theta,\kappa}}}\mathbf{E}_{a\sim\pi^*_{\theta,\kappa}(\cdot\vert s)}[\nabla_\theta\log\pi_{\theta^k}(a\vert s)\cdot(\theta^{k+1}-\theta^k)]-\frac{M}{2}\Vert\theta^{k+1}-\theta^k\Vert^2\\
	=&\eta_1\mathbf{E}_{s\sim d_\rho^{\pi^*_{\theta,\kappa}}}\mathbf{E}_{a\sim\pi^*_{\theta,\kappa}(\cdot\vert s)}[\nabla_\theta\log\pi_{\theta^k}(a\vert s)\cdot\omega^k]-\frac{M\eta_1^2}{2}\Vert\omega^k\Vert^2\\
	=&\eta_1\mathbf{E}_{s\sim d_\rho^{\pi^*_{\theta,\kappa}}}\mathbf{E}_{a\sim\pi^*_{\theta,\kappa}(\cdot\vert s)}[\nabla_\theta\log\pi_{\theta^k}(a\vert s)\cdot\omega_*^k]+\eta\mathbf{E}_{s\sim d_\rho^{\pi^*_{\theta,\kappa}}}\mathbf{E}_{a\sim\pi^*_{\theta,\kappa}(\cdot\vert s)}[\nabla_\theta\log\pi_{\theta^k}(a\vert s)\cdot(\omega^k-\omega_*^k)]-\frac{M\eta_1^2}{2}\Vert\omega^k\Vert^2\\
	=&\eta_1[J_L(\pi^*_{\theta,\kappa},\bblambda^k)-J_L(\pi_\theta^k,\bblambda^k)]+\eta\mathbf{E}_{s\sim d_\rho^{\pi^*_{\theta,\kappa}}}\mathbf{E}_{a\sim\pi^*_{\theta,\kappa}(\cdot\vert s)}[\nabla_\theta\log\pi_{\theta^k}(a\vert s)\cdot\omega_*^k]-\eta[J_L(\pi^*_{\theta,\kappa},\bblambda^k)-J_L(\pi_\theta^k,\bblambda^k)]\\
	&+\eta_1\mathbf{E}_{s\sim d_\rho^{\pi^*_{\theta,\kappa}}}\mathbf{E}_{a\sim\pi^*_{\theta,\kappa}(\cdot\vert s)}[\nabla_\theta\log\pi_{\theta^k}(a\vert s)\cdot(\omega^k-\omega_*^k)]-\frac{M\eta_1^2}{2}\Vert\omega^k\Vert^2\\	\overset{(b)}=&\eta_1[J_L(\pi^*_{\theta,\kappa},\bblambda^k)-J_L(\pi_\theta^k,\bblambda^k)]+\frac{\eta}{1-\gamma}\mathbf{E}_{s\sim d_\rho^{\pi^*_{\theta,\kappa}}}\mathbf{E}_{a\sim\pi^*_{\theta,\kappa}(\cdot\vert s)}\bigg[\nabla_\theta\log\pi_{\theta^k}(a\vert s)\cdot(1-\gamma)\omega_*^k-A_{L,\bblambda}^{\pi_{\theta^k}}(s,a)\bigg]\\
	&+\eta_1\mathbf{E}_{s\sim d_\rho^{\pi^*_{\theta,\kappa}}}\mathbf{E}_{a\sim\pi^*_{\theta,\kappa}(\cdot\vert s)}[\nabla_\theta\log\pi_{\theta^k}(a\vert s)\cdot(\omega^k-\omega_*^k)]-\frac{M\eta_1^2}{2}\Vert\omega^k\Vert^2\\
	\overset{(c)}\geq&\eta_1[J_L(\pi^*_{\theta,\kappa},\bblambda^k)-J_L(\pi_\theta^k,\bblambda^k)]-\frac{\eta_1}{1-\gamma}\sqrt{\mathbf{E}_{s\sim d_\rho^{\pi^*_{\theta,\kappa}}}\mathbf{E}_{a\sim\pi^*_{\theta,\kappa}(\cdot\vert s)}\bigg[\bigg(\nabla_\theta\log\pi_{\theta^k}(a\vert s)\cdot(1-\gamma)\omega_*^k-A_{L,\bblambda^k}^{\pi_{\theta^k}}(s,a)\bigg)^2\bigg]}\\
	&-\eta_1\mathbf{E}_{s\sim d_\rho^{\pi^*_{\theta,\kappa}}}\mathbf{E}_{a\sim\pi^*_{\theta,\kappa}(\cdot\vert s)}\Vert\nabla_\theta\log\pi_{\theta^k}(a\vert s)\Vert_2\Vert(\omega^k-\omega_*^k)\Vert_2-\frac{M\eta_1^2}{2}\Vert\omega^k\Vert^2\\
	\overset{(d)}\geq&\eta_1[J_L(\pi^*_{\theta,\kappa},\bblambda^k)-J_L(\pi_\theta^k,\bblambda^k)]-\frac{\eta_1\sqrt{\epsilon_{bias}}}{1-\gamma}-\eta_1 G\Vert(\omega^k-\omega_*^k)\Vert_2-\frac{M\eta_1^2}{2}\Vert\omega^k\Vert^2
	\end{aligned}	
	\end{equation}
	where the step (a) holds by Assumption \ref{ass_score} and step (b) holds by Performance Difference Lemma \cite{Kakade2002}. Step (c) uses the convexity of the function $f(x)=x^2$ and Cauchy's inequality. Step (d) follows from the Assumption \ref{ass_transfer_error}. Rearranging items, we have
	\begin{equation}
		\begin{split}
		J_L(\pi^*_{\theta,\kappa},\bblambda^k)-J_L(\pi_\theta^k,\bblambda^k)&\leq \frac{\sqrt{\epsilon_{bias}}}{1-\gamma}+ G\Vert(\omega^k-\omega_*^k)\Vert_2+\frac{M\eta_1^2}{2}\Vert\omega^k\Vert^2\\
		&+\frac{1}{\eta_1}\mathbf{E}_{s\sim d_\rho^{\pi^*_{\theta,\kappa}}}[KL(\pi^*_{\theta,\kappa}(\cdot\vert s)\Vert\pi_{\theta^k}(\cdot\vert s))-KL(\pi^*_{\theta,\kappa}(\cdot\vert s)\Vert\pi_{\theta^{k+1}}(\cdot\vert s))]
		\end{split}
	\end{equation}
	Summing from $k=0$ to $K-1$ and dividing by $K$, we have
	\begin{equation}
		\begin{split}
		\frac{1}{K}\sum_{k=1}^{K}\bigg(J_L(\pi^*_{\theta,\kappa},\bblambda^k)-J_L(\pi_\theta^k,\bblambda^k)\bigg)&\leq \frac{\sqrt{\epsilon_{bias}}}{1-\gamma}+\frac{M\eta_1}{2K}\sum_{k=0}^{K-1}\Vert\omega^k\Vert^2+ \frac{G}{K}\sum_{k=0}^{K-1}\Vert(\omega^k-\omega_*^k)\Vert_2\\
		&+\frac{1}{\eta_1 K}\mathbf{E}_{s\sim d_\rho^{\pi^*_{\theta,\kappa}}}[KL(\pi^*_{\theta,\kappa}(\cdot\vert s)\Vert\pi_{\theta^0}(\cdot\vert s))-KL(\pi^*_{\theta,\kappa}(\cdot\vert s)\Vert\pi_{\theta^{K}}(\cdot\vert s))]
		\end{split}
	\end{equation}
	Taking the expectation with respect to $\theta^k,k=0,1, \cdots, K-1$ and noticing that KL divergence is bounded by $\log(|\ccalA|)$, we have
	\begin{equation}\label{eq:bound_JL}
	    \frac{1}{K}\sum_{k=1}^{K}\mathbf{E}\bigg(J_L(\pi^*_{\theta,\kappa},\bblambda^k)-J_L(\pi_\theta^k,\bblambda^k)\bigg)\leq \frac{\sqrt{\epsilon_{bias}}}{1-\gamma}+\frac{M\eta_1}{2K}\sum_{k=0}^{K-1}\mathbf{E}\Vert\omega^k\Vert^2+\frac{\log(|\ccalA|)}{\eta_1 K}+\frac{G}{K}\sum_{k=1}^{K}\mathbf{E}\Vert(\omega^k-\omega_*^k)\Vert_2
	\end{equation}
\end{proof}

\subsection{Bound on the difference between $\omega^k$ and $\omega_*^k$}\label{sec:bound1}
Equipped with Lemma \ref{lem:bound_omega_diff}, we are ready to bound $\frac{G}{K}\sum_{k=0}^{K-1}\Vert(\omega^k-\omega_*^k)\Vert_2$. Using the Jensen inequality twice, we have
	\begin{equation}
	    \begin{aligned}
	    \bigg(\frac{1}{K}\sum_{k=0}^{K-1}\mathbf{E}\Vert \omega^k-\omega_*^k\Vert_2\bigg)^2&\leq \frac{1}{K}\sum_{k=0}^{K-1}\bigg(\mathbf{E}\Vert \omega^k-\omega_*^k\Vert_2\bigg)^2\leq \frac{1}{K}\sum_{k=0}^{K-1}\mathbf{E}[\Vert \omega^k-\omega_*^k\Vert_2^2]\\
	    &\leq \frac{4}{N\mu_F}\bigg[2[\frac{G^2(1+I\Lambda)}{\mu_F(1-\gamma)^2}+\frac{2}{(1-\gamma)^2}]\sqrt{d}+\frac{G^2(1+I\Lambda)}{\mu_F(1-\gamma)^2}\bigg]^2
	    \end{aligned}
	\end{equation}
	Thus,
	\begin{equation}\label{eq:grad_bias_final_bound}
	    \frac{1}{K}\sum_{k=0}^{K-1}\mathbf{E}\Vert \omega^k-\omega_*^k\Vert_2\leq \frac{2}{\sqrt{N\mu_F}}\bigg[2[\frac{G^2(1+I\Lambda)}{\mu_F(1-\gamma)^2}+\frac{2}{(1-\gamma)^2}]\sqrt{d}+\frac{G^2(1+I\Lambda)}{\mu_F(1-\gamma)^2}\bigg]
	\end{equation}
	\subsection{Bound on the norm of $\omega^k$}\label{sec:bound2}
	By the Cauchy's Inequality, we have
	\begin{equation}\label{eq:grad_norm_final_bound}
	    \begin{aligned}
	    \frac{1}{K}\sum_{k=0}^{K-1}\mathbf{E}[\Vert\omega^k\Vert_2^2]&\leq \frac{2}{K}\sum_{k=0}^{K-1}\mathbf{E}[\Vert\omega^k-\omega_*^k\Vert_2^2]+\frac{2}{K}\sum_{k=0}^{K-1}\mathbf{E}[\Vert\omega_*^k\Vert_2^2]\\
	     &\overset{(a)}\leq \frac{2}{K}\sum_{k=0}^{K-1}\mathbf{E}[\Vert\omega^k-\omega_*^k\Vert_2^2]+\frac{2}{\mu_F^2K}\sum_{k=0}^{K-1}\mathbf{E}[\Vert \nabla_\theta J_L(\theta^k,\bblambda^k)\Vert_2^2]\\
	     &\overset{(b)}\leq \frac{4}{N\mu_F}\bigg[2[\frac{G^2(1+I\Lambda)}{\mu_F(1-\gamma)^2}+\frac{2}{(1-\gamma)^2}]\sqrt{d}+\frac{G^2(1+I\Lambda)}{\mu_F(1-\gamma)^2}\bigg]^2+\frac{16G^4L_J}{\mu_F^4(1-\gamma)}[\frac{2(1+I\Lambda)}{K}+\frac{I\Lambda }{\eta_2(1-\gamma)}]\\
	     &+\frac{16(\mu_F^2+2G^4)}{\mu_F^3 N}\bigg[2[\frac{G^2(1+I\Lambda)}{\mu_F(1-\gamma)^2}+\frac{2}{(1-\gamma)^2}]\sqrt{d}+\frac{G^2(1+I\Lambda)}{\mu_F(1-\gamma)^2}\bigg]^2
	    \end{aligned}
	\end{equation}
	where step (a) holds by Assumption \ref{ass_pd} and step (b) is true due to Lemma \ref{lem:bound_omega_diff} and Lemma \ref{lem:first_order}.
	\subsection{Final bound for $J_L$}
    Combining Eq. \eqref{eq:grad_bias_final_bound}, \eqref{eq:grad_norm_final_bound}, and \eqref{eq:bound_JL}, we have
    \begin{equation}\label{eq:bound_JL_final}
        \begin{aligned}
        \frac{1}{K}&\sum_{k=1}^{K}\mathbf{E}\bigg(J_L(\pi^*_{\theta,\kappa},\bblambda^k)-J_L(\pi_\theta^k,\bblambda^k)\bigg)\leq \frac{\sqrt{\epsilon_{bias}}}{1-\gamma}+\frac{M\eta_1}{2K}\sum_{k=0}^{K-1}\mathbf{E}\Vert\omega^k\Vert^2+\frac{\log(|\ccalA|)}{\eta_1 K}+\frac{G}{K}\sum_{k=1}^{K}\mathbf{E}\Vert(\omega^k-\omega_*^k)\Vert_2\\
        &\leq \frac{\sqrt{\epsilon_{bias}}}{1-\gamma}+\frac{\log(|\ccalA|)}{\eta_1 K}+\frac{2M\eta_1}{N\mu_F}\bigg[2[\frac{G^2(1+I\Lambda)}{\mu_F(1-\gamma)^2}+\frac{2}{(1-\gamma)^2}]\sqrt{d}+\frac{G^2(1+I\Lambda)}{\mu_F(1-\gamma)^2}\bigg]^2\\
	    &+\frac{2G}{\sqrt{N\mu_F}}\bigg[2[\frac{G^2(1+I\Lambda)}{\mu_F(1-\gamma)^2}+\frac{2}{(1-\gamma)^2}]\sqrt{d}+\frac{G^2(1+I\Lambda)}{\mu_F(1-\gamma)^2}\bigg]+\frac{8M\eta_1 G^4L_J}{\mu_F^4(1-\gamma)}[\frac{2(1+I\Lambda)}{K}+\frac{\eta_2I }{(1-\gamma)}]\\
	    &+\frac{8M\eta_1(\mu_F^2+2G^4)}{\mu_F^3 N}\bigg[2[\frac{G^2(1+I\Lambda)}{\mu_F(1-\gamma)^2}+\frac{2}{(1-\gamma)^2}]\sqrt{d}+\frac{G^2(1+I\Lambda)}{\mu_F(1-\gamma)^2}\bigg]^2
        \end{aligned}
    \end{equation}
    Recall $\eta_1=\frac{\mu_F^2}{4G^2L_J}$ and $\eta_2=\frac{1}{\sqrt{K}}$, the above equation can be simplified as 
    \begin{equation}
        \begin{aligned}
        &\frac{1}{K}\sum_{k=1}^{K}\mathbf{E}\bigg(J_L(\pi^*_{\theta,\kappa},\bblambda^k)-J_L(\pi_\theta^k,\bblambda^k)\bigg)\leq 
        \frac{\sqrt{\epsilon_{bias}}}{1-\gamma}+\frac{4G^2L_J\log(|\ccalA|)}{\mu_F^2 K}\\
        &+\frac{M(5\mu_F+8\frac{G^4}{\mu_F})}{2G^2L_J N}\bigg[2[\frac{G^2(1+I\Lambda)}{\mu_F(1-\gamma)^2}+\frac{2}{(1-\gamma)^2}]\sqrt{d}+\frac{G^2(1+I\Lambda)}{\mu_F(1-\gamma)^2}\bigg]^2\\
        &+\frac{2G}{\sqrt{N\mu_F}}\bigg[2[\frac{G^2(1+I\Lambda)}{\mu_F(1-\gamma)^2}+\frac{2}{(1-\gamma)^2}]\sqrt{d}+\frac{G^2(1+I\Lambda)}{\mu_F(1-\gamma)^2}\bigg]+\frac{2MG^2}{\mu_F^2(1-\gamma)}[\frac{2(1+I\Lambda)}{K}+\frac{I }{\sqrt{K}(1-\gamma)}]
        \end{aligned}
    \end{equation}
    Define
    \begin{equation}
        \eps_{K,N}:=\ccalO\bigg(\frac{L_J}{K}\bigg)+\ccalO\bigg(\frac{I^2\Lambda^2}{L_J N(1-\gamma)^4}\bigg)+\ccalO\bigg(\frac{I\Lambda}{\sqrt{N}(1-\gamma)^2}\bigg)+\ccalO\bigg(\frac{I\Lambda}{K(1-\gamma)}\bigg)+\ccalO\bigg(\frac{I}{\sqrt{K}(1-\gamma)^2}\bigg)
    \end{equation}
    Recalling $L_J=\frac{M}{(1-\gamma)^2}+\frac{2G^2}{(1-\gamma)^3}$, the above definition can be simplified as 
    \begin{equation}\label{eq:defintion_epsKN}
        \eps_{K,N}=\ccalO\bigg(\frac{1}{(1-\gamma)^3K}\bigg)+\ccalO\bigg(\frac{I^2\Lambda^2}{ (1-\gamma)^2 N}\bigg)+\ccalO\bigg(\frac{I\Lambda}{(1-\gamma)\sqrt{N}}\bigg)+\ccalO\bigg(\frac{I\Lambda}{K(1-\gamma)}\bigg)+\ccalO\bigg(\frac{I}{\sqrt{K}(1-\gamma)^2}\bigg)
    \end{equation}
    Finally, we have
    \begin{equation}\label{eq:bound_JL_final_simplified}
        \frac{1}{K}\sum_{k=1}^{K}\mathbf{E}\bigg(J_L(\pi^*_{\theta,\kappa},\bblambda^k)-J_L(\pi_\theta^k,\bblambda^k)\bigg)\leq\frac{\sqrt{\epsilon_{bias}}}{1-\gamma}+\epsilon_{K,N}
    \end{equation}
\section{Bounding the Gap Between original and Conservative Problem}
    \subsection{Proof of Lemma \ref{lem:duality}}\label{sec:app_lemma5}
    \begin{proof}
        Define for any class of policies in the set $\Pi_\theta$
        \begin{equation}
            \bm{L}=\cup_{\pi\in\Pi_\theta} d^{\pi}(s,a)
        \end{equation}
        For any $\theta\in\Theta$, it is proved in \cite{altman1999constrained}[Eq. 3.4]
        \begin{equation}\label{eq:lambda}
            \sum_{a\in\ccalA}d^{\pi_\theta}(s,a)=\rho(s)(1-\gamma)+\gamma\sum_{s'\in\ccalS}\sum_{a'\in\ccalA}d^{\pi_\theta}(s',a')\mathbf{P}(s|s',a')
        \end{equation}
        which means that $d^{\pi_\theta}$ satisfies Eq. \eqref{eq:occupancy_set} and thus $\bm{L}\subset d$, which gives $J_r^{\pi_\theta^*}\leq \left<r, \phi^*\right>$. 
        \\For the other side, define $\phi(s)=\sum_{a\in\ccalA}\phi(s,a)$ and recall the definition of $\phi(s,a)$,
        \begin{equation}
            \phi(s)=\rho(s)(1-\gamma)+\gamma\sum_{s'\in\ccalS}\sum_{a'\in\ccalA}\phi(s',a')\mathbf{P}(s|s',a')\\
        \end{equation}
         Denote $\rho_\pi(s,a)=\rho(s)\pi(a|s)$ and $\mathbf{P}_{\pi}(s,a;s',a')=\mathbf{P}_{\pi}(s|s',a')\pi(a|s)$. Multiplying both sides by $\pi'(a|s)=\frac{\phi(s,a)}{\sum_a\phi(s,a)}$ and using Assumption \ref{ass:sufficient_para}, we have
        \begin{equation}
            \begin{aligned}
            \phi(s,a)&=\rho_{\pi'}(s,a)(1-\gamma)+\gamma\sum_{s'\in\ccalS}\sum_{a'\in\ccalA}\phi(s',a')\pi'(a|s)\mathbf{P}(s|s',a')\\
            &=\rho_{\pi'}(s,a)(1-\gamma)+\gamma\sum_{s'\in\ccalS}\sum_{a'\in\ccalA}\phi(s',a')\bigg(\pi'(a|s)-\pi_\theta(a|s)\bigg)\mathbf{P}(s|s',a')+\gamma\sum_{s'\in\ccalS}\sum_{a'\in\ccalA}\phi(s',a')\pi_\theta(a|s)\mathbf{P}(s|s',a')\\\\
            &\leq \rho_{\pi_\theta}(s,a)(1-\gamma)+\rho(s)\eps_{bias2}(1-\gamma)+\eps_{bias2}\gamma\sum_{s'\in\ccalS}\sum_{a'\in\ccalA}\phi(s',a')\mathbf{P}(s|s',a')+\gamma\sum_{s'\in\ccalS}\sum_{a'\in\ccalA}\phi(s',a')\pi_{\theta}(a|s)\mathbf{P}(s',a';s,a)\\
            &\overset{(a)}\leq \rho_{\pi_\theta}(s,a)(1-\gamma)+\eps_{bias2}(1-\gamma)+\eps_{bias2}\gamma+\gamma\sum_{s'\in\ccalS}\sum_{a'\in\ccalA}\phi(s',a')\mathbf{P}_{\pi_\theta}(s',a';s,a)\\
            &\leq \rho_{\pi_\theta}(s,a)(1-\gamma)+\gamma\sum_{s'\in\ccalS}\sum_{a'\in\ccalA}\phi(s',a')\mathbf{P}_{\pi_\theta}(s',a';s,a)+\eps_{bias2}\\
            \end{aligned}
        \end{equation}
    	where step (a) holds becuase $\rho(s),\mathbf{P}(s|s',a')\leq 1,\forall s\in\ccalS$ and $\sum_{s',a'}\phi(s',a')=1$. Define $\mathbf{P}_{\pi_\theta}\in\mathbb{R}^{|\ccalS||\ccalA|\times |\ccalS||\ccalA|}$ as the transition matrix that we transit from state-action pair $(s',a')$ to $(s,a)$ following $\pi_\theta$ and denote $\bbphi=[\phi(s_1),\cdots,\phi(s_{|\ccalS|})]^T$, then the above equation can be written compactly as
        \begin{equation}
            (\mathbf{I}-\gamma\mathbf{P}_{\pi_\theta})\bbphi\leq \bbrho(1-\gamma)+\eps_{bias2}\cdot \mathbf{1}
        \end{equation}
        Note that all the eigenvalues of $(\mathbf{I}-\gamma\mathbf{P}_{\pi_\theta})$ are non-zero,  thus $(\mathbf{I}-\gamma\mathbf{P}_{\pi_\theta})$ is invertible,
        \begin{equation}\label{eq:bound_dualgap1}
            \bbphi\leq (1-\gamma)\rho(\mathbf{I}-\gamma\mathbf{P}_{\pi_\theta})^{-1}+\eps_{bias2}(\mathbf{I}-\gamma\mathbf{P}_{\pi_\theta})^{-1} \cdot\mathbf{1}
        \end{equation}
        Expanding $(\mathbf{I}-\gamma\mathbf{P}_{\pi_\theta})^{-1}$ as $\mathbf{I}+\gamma\mathbf{P}_{\pi_\theta}+\gamma^2\mathbf{P}_{\pi_\theta}^2+\cdots$, we have
        \begin{equation}\label{eq:bound_dualgap2}
            (1-\gamma)\bm{\rho}(\mathbf{I}-\gamma\mathbf{P}_{\pi_\theta})^{-1}=\bm{d}^{\pi_\theta}\quad \text{and}\quad (\mathbf{I}-\gamma\mathbf{P}_{\pi_\theta})^{-1}\cdot \mathbf{1}\leq \frac{1}{1-\gamma}
        \end{equation}
        Substituting Eq. \eqref{eq:bound_dualgap2} in Eq. \eqref{eq:bound_dualgap1}, we get
        \begin{equation}
            \bbphi\leq \bm{d}^{\pi_\theta}+\frac{\eps_{bias2}}{1-\gamma}\cdot\mathbf{1}
        \end{equation}
        Thus,
        \begin{equation}
            \frac{\left<\bbphi,\bbr\right>}{1-\gamma}\leq \frac{\left<\bm{d}^{\pi_\theta},\bbr\right>}{1-\gamma}+\frac{\eps_{bias2}}{(1-\gamma)^2}
        \end{equation}
        which means for any $\bbphi\in \ccalD$, there exists a $\pi_\theta,\theta\in\Theta$, such that
        \begin{equation}
            \frac{\left<\bbphi^{\pi_\theta},\bbr\right>}{1-\gamma}\leq J_r^{\pi_\theta}+\frac{\eps_{bias2}}{(1-\gamma)^2}
        \end{equation}
    \end{proof}
    \subsection{Proof of Lemma \ref{lem:bound_conservative}}\label{sec:app_lemma6}
    \begin{proof}
	    Define $\theta^*$ as the optimal solution to the original problem and define $\bbphi^*=\bbphi(\theta^*)=\bbphi^{\pi_{\theta^*}}$ as the corresponding optimal occupancy measure, which gives
	    \begin{equation}
	        \frac{\left<\bbphi^*,\bbg^i\right>}{1-\gamma}\geq 0
	    \end{equation}
	    Further, under the Slater Condition, Assumption \ref{ass_slater}, there exists at least one occupancy measure $\tilde{\bbphi}$ such that
	    \begin{equation}
	        \frac{\left<\tilde{\bbphi},\bbg^i\right>}{1-\gamma}\geq \varphi
	    \end{equation}
	    Define a new occupancy measure $\hat{\bbphi}=(1-\frac{\kappa}{\varphi})\bbphi^*+\frac{\kappa}{\varphi}\tilde{\bbphi}$. It can be shown as a feasible occupancy measure to the conservative version of new problem \eqref{eq:new_problem}. First, by the above equation, we have
	    \begin{equation}
        	\frac{\left<\hat{\bbphi},\bbg^i\right>}{1-\gamma}=\frac{1}{1-\gamma}\left<(1-\frac{\kappa}{\varphi})\bbphi^*+\frac{\kappa}{\varphi}\tilde{\bbphi},\bbg^i\right>\geq\frac{\kappa}{\varphi}\varphi=\kappa
	    \end{equation}
	    Second, due to $\hat{\bbphi}$ being the linear combination of two feasible occupancy measure, we have
	    \begin{equation}
        	\sum_{s'\in\ccalS}\sum_{a\in\ccalA}\hat{\phi}(s',a)(\delta_s(s')-\gamma \mathbf{P}(s|s',a))=(1-\gamma)\rho(s)
    	\end{equation}
    	Similarly, define $\theta_\kappa^*$ as the optimal solution to the original conservative problem and define $\bbphi_\kappa^*$ as the corresponding optimal occupancy measure. By Lemma \ref{lem:duality}, we have $J_r^{\pi_{\theta^*}}\leq \frac{\left<\bbphi^*,\bbr\right>}{1-\gamma}$ and $J_r^{\pi_{\theta_\kappa^*}}\geq \frac{\left<\bbphi^*_{\bbkappa},\bbr\right>}{1-\gamma}-\frac{\eps_{bias2}}{(1-\gamma)^2}$. Then, we can bound the optimal objective between original problem and conservative problem
    	\begin{equation}\label{eq:conservative_diff}
	        \begin{aligned}
	        J_r^{\pi_{\theta^*}}-J_r^{\pi_{\theta_\kappa^*}}&\leq \frac{\eps_{bias2}}{(1-\gamma)^2}+ \frac{\left<\bbphi^*,\bbr\right>}{1-\gamma}-\frac{\left<\bbphi^*_{\bbkappa},\bbr\right>}{1-\gamma}\\
	        &\overset{(a)}\leq\frac{\eps_{bias2}}{(1-\gamma)^2}+\frac{\left<\bbphi^*,\bbr\right>-\left<\hat{\bbphi},\bbr\right>}{1-\gamma}\\ &=\frac{\eps_{bias2}}{(1-\gamma)^2}+\frac{1}{1-\gamma}\left<\bbphi^*-(1-\frac{\kappa}{\varphi})\bbphi^*-\frac{\kappa}{\varphi}\tilde{\bbphi},\bbr\right>\\
	        &=\frac{\eps_{bias2}}{(1-\gamma)^2}+\frac{1}{1-\gamma}\left<\frac{\kappa}{\varphi}\bbphi^*-\frac{\kappa}{\varphi}\tilde{\bbphi},\bbr\right>\\
	        &\overset{(b)}\leq\frac{\eps_{bias2}}{(1-\gamma)^2}+\frac{1}{1-\gamma} \left<\frac{\kappa}{\varphi}\bbphi^*,\bbr\right>\overset{(c)}\leq \frac{\eps_{bias2}}{(1-\gamma)^2}+\frac{\kappa}{(1-\gamma)\varphi}\\
	        \end{aligned}
	    \end{equation}
	
	{The first step (a) holds because $\bbphi^*_{\bbkappa}$ is the optimal solution of the conservative problem, which gives larger value function than any other feasible occupancy measure. We drop the negative term in the step (b) and the last step (c) is true because $\left<\bbphi^*,\bbr\right>\leq 1$ by the definition of reward.}
\end{proof}

\section{Proof of Theorem \ref{main_theorem}}

\subsection{Analysis of Objective}
    Recall the definition of $J_L(\pi_\theta,\bblambda)=J_r(\pi_\theta)+\eta\sum_{i\in[I]}(J_{g^i}(\pi_\theta)-\kappa)$, we have
	\begin{equation}
        \label{eq:JL_converge}
		\frac{1}{K}\sum_{k=0}^{K-1}\mathbf{E}\bigg(J_r(\pi^*_{\theta,\kappa})-J_r(\pi_\theta^k)\bigg)\leq \frac{\sqrt{\epsilon_{bias}}}{1-\gamma}+\epsilon_{K,N}-	\frac{1}{K}\sum_{k=0}^{K-1}\sum_{i\in[I]}\mathbf{E}\bigg[\lambda_i^k\bigg(J_{g^i}(\pi^*_{\theta,\kappa})-J_{g^i}(\pi_\theta^k)\bigg)\bigg]
	\end{equation}
	Thus, we  need to find a bound the last term in above equation.
	\begin{equation*}\label{eq:bound_lambdak}
		\begin{aligned}
		0\leq (\lambda_i^K)^2&=\sum_{k=0}^{K-1}\bigg((\lambda^{k+1}_i)^2-(\lambda_i^k)^2\bigg)\\
		&=\sum_{k=0}^{K-1}\bigg(\max\big(0,\lambda_i^k-\eta_2(\hat{J}_{g^i}(\pi_\theta^k)-\kappa)\big)^2-(\lambda_i^k)^2\bigg)\\
		&\leq\sum_{k=0}^{K-1}\bigg(\big(\lambda_i^k-\eta_2(\hat{J}_{g^i}(\pi_\theta^k)-\kappa)\big)^2-(\lambda_i^k)^2\bigg)\\
		&=2\eta_2\sum_{k=0}^{K-1}\lambda_i^k(\kappa- \hat{J}_{g^i}(\pi_\theta^k))+\eta_2^2\sum_{k=0}^{K-1}(\hat{J}_{g^i}(\pi_\theta^k)-\kappa)^2\\
		&\overset{(a)}\leq 2\eta_2\sum_{k=0}^{K-1}\lambda_i^k(J_{g^i}(\pi^*_{\theta,\kappa})- \hat{J}_{g^i}(\pi_\theta^k))+\eta_2^2\sum_{k=0}^{K-1}(\hat{J}_{g^i}(\pi_\theta^k)-\kappa)^2\\
		&\leq 2\eta_2\sum_{k=0}^{K-1}\lambda_i^k(J_{g^i}(\pi^*_{\theta,\kappa})- \hat{J}_{g^i}(\pi_\theta^k))+2\eta_2^2\sum_{k=0}^{K-1}(\hat{J}_{g^i}(\pi_\theta^k))^2+2\eta_2^2\kappa^2K\\
		&= 2\eta_2\sum_{k=0}^{K-1}\lambda_i^k(J_{g^i}(\pi^*_{\theta,\kappa})- J_{g^i}(\pi_\theta^k))+2\eta_2\sum_{k=0}^{K-1}\lambda_i^k(J_{g^i}(\pi^k_{\theta})- \hat{J}_{g^i}(\pi_\theta^k))+2\eta_2^2\sum_{k=0}^{K-1}(\hat{J}_{g^i}(\pi_\theta^k))^2+2\eta_2^2\kappa^2K
		\end{aligned}
	\end{equation*}
	Step (a) holds because $\pi^*_{\theta,\kappa}$ is a feasible policy for conservative problem.
	Rearranging items and taking expectation, we have
	\begin{equation}
	    -\frac{1}{K}\sum_{k=0}^{K-1}\mathbf{E}\bigg[\lambda_i^k(J_{g^i}(\pi^*_{\theta,\kappa})- J_{g^i}(\pi_\theta^k))\bigg]\leq \frac{1}{K}\sum_{k=0}^{K-1}\mathbf{E}\bigg[\lambda_i^k(J_{g^i}(\pi^k_{\theta})- \hat{J}_{g^i}(\pi_\theta^k))\bigg]+\frac{\eta_2}{K}\sum_{k=0}^{K-1}\mathbf{E}[\hat{J}_{g^i}(\pi_\theta^k)]^2+\eta_2\kappa^2
	\end{equation}
	Notice that $\lambda_i^k$ and $\hat{J}_{g^i}(\pi_\theta^k)$ are independent at time $k$ and thus $\hat{J}_{g^i}$ is the unbiased estimator for $J_{g^i}$. Thus,
	\begin{equation}
	    \begin{aligned}
	    -\frac{1}{K}\sum_{k=0}^{K-1}\mathbf{E}\bigg[\lambda_i^k(J_{g^i}(\pi^*_{\theta,\kappa})- J_{g^i}(\pi_\theta^k))\bigg]&\leq \frac{\eta_2}{K}\sum_{k=0}^{K-1}\mathbf{E}[\hat{J}_{g^i}(\pi_\theta^k)]^2+\eta_2\kappa^2\\
	    &\overset{(a)}\leq \frac{\eta_2}{2(1-\gamma)^2}\frac{K+1}{K}+\eta_2\kappa^2\\
	    &\overset{(b)}\leq \frac{\eta_2}{(1-\gamma)^2}\bigg(1+\frac{K+1}{2K}\bigg)\\
	    &\leq \frac{2\eta_2}{(1-\gamma)^2}
	    \end{aligned}
	\end{equation}
	where step (a) holds by \cite{Ding2020}[Appendix G] and step (b) is true because $\kappa\leq \frac{1}{1-\gamma}$. Combining with Eq. \eqref{eq:bound_JL_final_simplified} and recall $\eta_2=\frac{1}{\sqrt{K}}$, we have
	\begin{equation}
		\frac{1}{K}\sum_{k=0}^{K-1}\bigg(J_r(\pi^*_{\theta,\kappa})-J_r(\pi_\theta^k)\bigg)\leq \frac{\sqrt{\epsilon_{bias}}}{1-\gamma}+\epsilon_{K,N} + \frac{2}{\sqrt{K}(1-\gamma)^2}
	\end{equation}
    Combining the above equation with Eq. \eqref{eq:conservative_diff} in lemma \ref{lem:duality}, we have the bound for objective
    \begin{equation}\label{eq:bound_objective}
        \begin{aligned}
        \frac{1}{K}\sum_{k=0}^{K-1}\bigg(J_r(\pi^*_{\theta})-J_r(\pi_\theta^k)\bigg)&=\frac{1}{K}\sum_{k=0}^{K-1}\bigg(J_r(\pi^*_{\theta})-J_r(\pi^*_{\theta,\kappa})\bigg)+\frac{1}{K}\sum_{k=0}^{K-1}\bigg(J_r(\pi^*_{\theta,\kappa})-J_r(\pi_\theta^k)\bigg)\\
        &\leq \frac{\sqrt{\epsilon_{bias}}}{1-\gamma}+\frac{\eps_{bias2}}{(1-\gamma)^2}+\frac{\kappa}{(1-\gamma)\varphi}+\epsilon_{K,N} + \frac{2}{\sqrt{K}(1-\gamma)^2}
        \end{aligned}
    \end{equation}
    To get the final bound, we need the definition of $\kappa$, which is defined in the following section.
    
    \subsection{Analysis of Constraint}
    For any $\lambda\in [0,
    \Lambda]$, given the dual update in algorithm \ref{alg:spdgd}, we have
\begin{equation}
    \begin{aligned}
	    \vert\lambda_i^{k+1} - \lambda\vert^2 
		&\overset{(a)}{\leq} \bigg|\lambda_i^{k} - \eta_2(\hat{J}_{g^i}(\pi_\theta^k)-\kappa)   -\lambda\bigg|^2\\
		&=\big|\lambda_i^{k} -\lambda\big|^2 -2\eta_2(\hat{J}_{g^i}(\pi_\theta^k)-\kappa)\big(\lambda_i^{k}  -\lambda\big) +\eta_2^2(\hat{J}_{g^i}(\pi_\theta^k)-\kappa)^2
		\\
		&\leq\big|\lambda_{k} -\lambda\big|^2 -2\eta_2(\hat{J}_{g^i}(\pi_\theta^k)-\kappa)\big(\lambda_i^{k}  -\lambda\big)  + 2\eta_2^2(\hat{J}_{g^i}(\pi_\theta^k)^2+\kappa^2)
	\end{aligned}
\end{equation}
where $(a)$ is because of the non-expansiveness of projection $\mathcal{P}_\Lambda$. Averaging the above inequality over $k=1,\ldots,K$ yields

\begin{equation}
    \begin{aligned}
        0\leq\frac{1}{K}\vert\lambda_i^{K+1} - \lambda\vert^2 
        \leq\frac{1}{K}\vert{{\lambda_i^1  -\lambda}}\vert^2 -\frac{2\eta_2}{K}\sum_{k=1}^{K} (\hat{J}_{g^i}(\pi_\theta^k)-\kappa)\big(\lambda_i^{k}  -\lambda\big) +\frac{2\eta_2^2}{K}\sum_{k=1}^{K}(\hat{J}_{g^i}(\pi_\theta^k)^2+\kappa^2)
    \end{aligned}
\end{equation}
Taking expectations at both sides, notice that $E[\hat{J}_g^i(\pi_\theta^k)]=J_g^i(\pi_\theta^k)$ and $\lambda_i^k$ is independent of $\hat{J}_g^i(\pi_\theta^k)]$
\begin{equation}
    \label{eq.removelamda}
    \mathbf{E}\bigg[\frac{1}{K}\sum_{k=1}^{K} (J_{g^i}(\pi_\theta^k)-\kappa)\big(\lambda_i^{k}  -\lambda\big)\bigg] \leq \frac{1}{2\eta_2 K}\big\vert{{\lambda_i^1  -\lambda}}\big\vert^2 + \frac{2\eta_2}{(1-\gamma)^2}
\end{equation}	
Notice that $\lambda_i^k (J_{g^i}(\pi^*_{\theta,\kappa})-\kappa)\geq 0$, $\forall k$. Adding the above inequality to \eqref{eq:JL_converge} at both sides, recall $\eta_2=\frac{1}{\sqrt{K}}$ we have,
\begin{equation}
    \label{eq:bound_Jc1}
    \begin{aligned}
        \frac{1}{K}\sum_{k=0}^{K-1}\mathbf{E}\bigg(J_r(\pi^*_{\theta,\kappa})-J_r(\pi_\theta^k)\bigg)+\lambda\mathbf{E}\bigg[\frac{1}{K}\sum_{i\in I}\sum_{k=1}^{K}\bigg(\kappa-J_{g^i}(\pi_\theta^k)\bigg)\bigg]\leq \frac{\sqrt{\epsilon_{bias}}}{1-\gamma}+\epsilon_{K,N} + \frac{1}{2\sqrt{K}}\sum_{i\in I}\big\vert{{\lambda_i^1  -\lambda}}\big\vert^2 + \frac{2I}{\sqrt{K}(1-\gamma)^2}
    \end{aligned}
\end{equation}
Setting $\lambda=\Lambda=\frac{2}{(1-\gamma)\varphi}$ if $\mathbf{E}\bigg[\frac{1}{K}\sum_{i\in I}\sum_{k=1}^{K}\bigg(\kappa-J_{g^i}(\pi_\theta^k)\bigg)\bigg]\geq 0$ and $\lambda=0$ otherwise, we have,
\begin{equation}
    \label{eq:bound_Jc2}
        \frac{1}{K}\sum_{k=0}^{K-1}\mathbf{E}\bigg(J_r(\pi^*_{\theta,\kappa})-J_r(\pi_\theta^k)\bigg)+\frac{2}{\varphi}\mathbf{E}\bigg[\frac{1}{K}\sum_{i\in I}\sum_{k=1}^{K}\bigg(\kappa-J_{g^i}(\pi_\theta^k)\bigg)\bigg]_+\leq \frac{\sqrt{\epsilon_{bias}}}{1-\gamma}+\epsilon_{K,N} + \frac{2I}{(1-\gamma)^2\varphi^2\sqrt{K}} + \frac{2I}{\sqrt{K}(1-\gamma)^2}
\end{equation}
We define a new policy $\bar{\pi}$ which uniformly chooses the policy $\pi_{\theta_k}$ for $k\in[K]$. By the occupancy measure method, $J_g(\theta_k)$ is linear in terms of an occupancy measure induced by policy $\pi_{\theta_k}$. Therefore,
\begin{equation}\label{eq_avg_value}
    \frac{1}{K}\sum_{k=1}^{K}J_r(\pi_\theta^k)=J_r^{\bar{\pi}}\quad \frac{1}{K}\sum_{k=1}^{K}J_{g^i}(\pi_\theta^k)=J_{g^i}^{\bar{\pi}}
\end{equation}
Injecting the above relation to \eqref{eq:bound_Jc2}, we have
\begin{equation}
    \mathbf{E}\bigg[J_r(\pi^*_{\theta,\kappa})-J_r^{\bar\pi}\bigg]+\frac{2}{(1-\gamma)\varphi}\sum_{i\in I}\mathbf{E}\bigg[\kappa-J_{g^i}^{\bar\pi}\bigg]_+\leq \frac{\sqrt{\epsilon_{bias}}}{1-\gamma}+\epsilon_{K,N}  + \frac{2I}{(1-\gamma)^2\varphi^2\sqrt{K}} + \frac{2I}{\sqrt{K}(1-\gamma)^2}
\end{equation}
By Lemma \ref{lem.constraint}, we arrive at,
\begin{equation}
     \sum_{i\in I}\mathbf{E}\bigg[\kappa-J_{g^i}^{\bar\pi}\bigg]_+\leq \varphi\sqrt{\epsilon_{bias}}+ \varphi(1-\gamma)\epsilon_{K,N}  + \frac{2I}{(1-\gamma)\varphi\sqrt{K}} + \frac{2\varphi I}{\sqrt{K}(1-\gamma)}
\end{equation}
By Jensen's inequality, we have
\begin{equation}
    \bigg[\kappa-\mathbf{E}[J_{g^i}^{\bar\pi}]\bigg]_+\leq \mathbf{E}\bigg[\kappa-J_{g^i}^{\bar\pi}\bigg]_+\leq \sum_{i\in I}\mathbf{E}\bigg[\kappa-J_{g^i}^{\bar\pi}\bigg]_+
\end{equation}
If $\mathbf{E}[J_{g^i}^{\bar\pi}]<0$, then, we have
\begin{equation}
    \mathbf{E}[J_{g^i}^{\bar\pi}]\geq \kappa - \bigg[\varphi\sqrt{\epsilon_{bias}}+ \varphi(1-\gamma)\epsilon_{K,N}  + \frac{2I}{(1-\gamma)\varphi\sqrt{K}} + \frac{2\varphi I}{\sqrt{K}(1-\gamma)}\bigg]
\end{equation}
To ensure $\mathbf{E}[J_{g^i}^{\bar\pi}]\geq 0$, we just need to set
\begin{equation}
    \label{eq:kappa_def}
    \kappa=\varphi\sqrt{\epsilon_{bias}}+ \varphi(1-\gamma)\epsilon_{K,N}  + \frac{2I}{(1-\gamma)\varphi\sqrt{K}} + \frac{2\varphi I}{\sqrt{K}(1-\gamma)}
\end{equation}

    \subsection{Final Bound on Objective}
    Here we restate Eq. \eqref{eq:bound_objective} as
    \begin{equation}
        \frac{1}{K}\sum_{k=0}^{K-1}\bigg(J_r(\pi^*_{\theta})-J_r(\pi_\theta^k)\bigg)
        \leq \frac{\sqrt{\epsilon_{bias}}}{1-\gamma}+\frac{\eps_{bias2}}{(1-\gamma)^2}+\frac{\kappa}{(1-\gamma)\varphi}+\epsilon_{K,N} + \frac{2}{\sqrt{K}(1-\gamma)^2}
    \end{equation}
    Substituting Eq. \eqref{eq:kappa_def} into the above equation 
    \begin{equation}
        \frac{1}{K}\sum_{k=0}^{K-1}\bigg(J_r(\pi^*_{\theta})-J_r(\pi_\theta^k)\bigg)
        \leq \frac{2\sqrt{\epsilon_{bias}}}{1-\gamma}+\frac{\eps_{bias2}}{(1-\gamma)^2}+ \frac{2I}{(1-\gamma)^2\varphi^2\sqrt{K}}+2\epsilon_{K,N} + \frac{2+2I}{\sqrt{K}(1-\gamma)^2}
    \end{equation}
    Recall that
	\begin{equation}
         \eps_{K,N}=\ccalO\bigg(\frac{1}{(1-\gamma)^3K}\bigg)+\ccalO\bigg(\frac{I^2\Lambda^2}{ (1-\gamma)^2 N}\bigg)+\ccalO\bigg(\frac{I\Lambda}{(1-\gamma)\sqrt{N}}\bigg)+\ccalO\bigg(\frac{I\Lambda}{K(1-\gamma)}\bigg)+\ccalO\bigg(\frac{I}{\sqrt{K}(1-\gamma)^2}\bigg)
    \end{equation}
    Recall the definition $\Lambda=\frac{2}{(1-\gamma)\varphi}, $we have the bound for objective as
    \begin{equation}
        \begin{aligned}
        \frac{1}{K}\sum_{k=0}^{K-1}\bigg(J_r(\pi^*_{\theta})-J_r(\pi_\theta^k)\bigg)
        &\leq \ccalO\bigg(\frac{\sqrt{\epsilon_{bias}}}{1-\gamma}\bigg)+\ccalO\bigg(\frac{\eps_{bias2}}{(1-\gamma)^2}\bigg)+\ccalO\bigg(\frac{1}{(1-\gamma)^3K}\bigg)+\ccalO\bigg(\frac{I^2\Lambda^2}{ (1-\gamma)^2 N}\bigg)\\
        &+\ccalO\bigg(\frac{I\Lambda}{(1-\gamma)\sqrt{N}}\bigg)+\ccalO\bigg(\frac{I}{\varphi^2\sqrt{K}(1-\gamma)^2}\bigg)
        \end{aligned}
    \end{equation}
    Taking $K=\ccalO\bigg(\frac{I^2}{\varphi^2(1-\gamma)^4\eps^2}\bigg)$ and $N=\ccalO\bigg(\frac{I^2\Lambda^2}{(1-\gamma)^2\eps^2}\bigg)$, we have
    \begin{equation}
        \frac{1}{K}\sum_{k=0}^{K-1}\bigg(J_r(\pi^*_{\theta})-J_r(\pi_\theta^k)\bigg)\leq \ccalO\bigg(\frac{\sqrt{\epsilon_{bias}}}{1-\gamma}\bigg)+\ccalO\bigg(\frac{\eps_{bias2}}{(1-\gamma)^2}\bigg)+\ccalO\bigg(\epsilon\bigg)
    \end{equation}

\section{Evaluation Details}\label{app:sim}
For the evaluations, we consider a random CMDP, where the number of elements of state space and action space are $|\ccalS|=10$ and $|\ccalA|=5$, respectively. Each entry of the transition matrix $P(s'|s,a)$ is generated uniformly at random in $[0,1]$, followed by normalization. The reward function is generated by uniform distribution $r(s,a)\sim U(0,1)$. Only 1 constraint function is considered and generated by $g(s,a)\sim U(-0.71, 0.29)$. The initial state distribution is set to uniform and the discount factor is $\gamma=0.8$. For the general parameterization, we use a feature map $\Phi^{|\ccalS||\ccalA|\times d}$ with dimension $d=35$ along with the softmax parametrization. Thus, the policy $\pi(a|s)$ can be expressed as $\pi(a|s)=\frac{\exp(\left<\Phi_{sa},\theta\right>)}{\sum_{a'\in\ccalA}\exp(\left<\Phi_{sa'},\theta\right>)}$, where $\Phi_{sa}$ is the corresponding row in matrix $\Phi$ to state and action pair $(s,a)$. For each SGD procedure, we use $N=100$ number of samples. The learning rate for $\theta$ and $\lambda$ are both set to 0.1\footnote{The simulation code is extended from the code in \cite{Ding2020}.}. To evaluate the average performance of the algorithm, we run 40 cases and for each case, we run the algorithm for $K=7000$ iterations. We evaluate and compare the average performance and standard deviation between the proposed algorithm with $\kappa=1$ and the NPG-PD algorithm \cite{Ding2020}, which doesn't consider the zero constraint violation case (equivalently $\kappa=0$) in  Figure \ref{fig:compare2}.

\begin{figure}[htbp]
    \centering
    \subfigure[Comparison for the convergence of objective]{
        \includegraphics[width=2.5in]{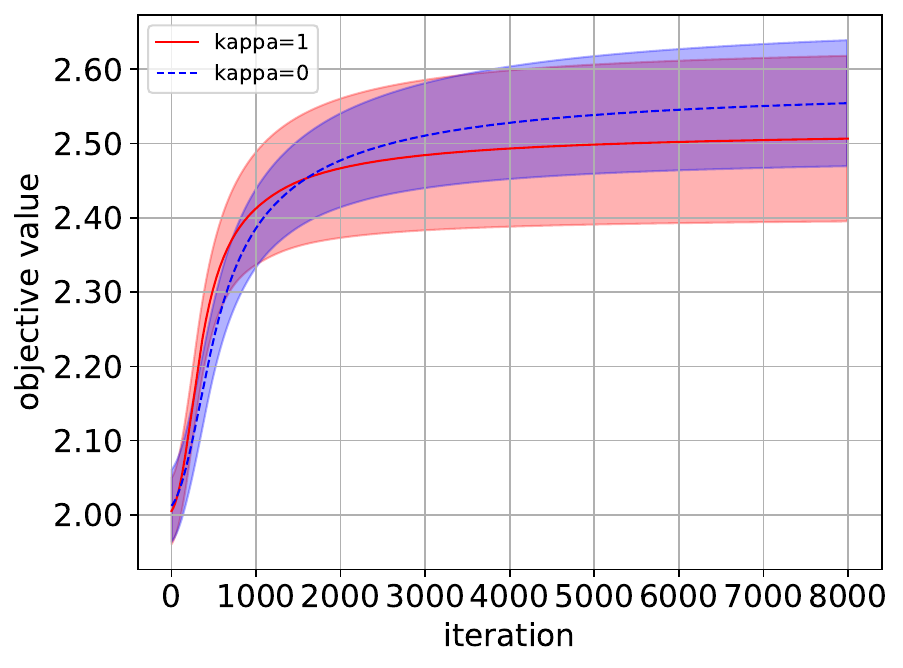}
    }
    \subfigure[Comparison for the constriant violation]{
	\includegraphics[width=2.5in]{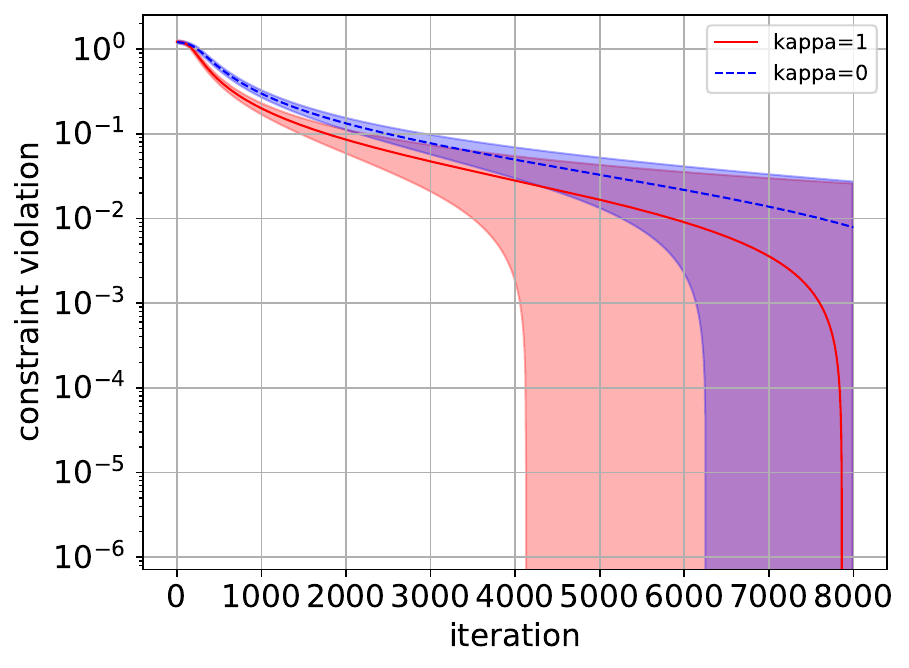}
    }
    \caption{Comparison of objective and constraint violation between the proposed algorithm $\kappa=1$ and NPD-PD \cite{Ding2020} $\kappa=0$. For the constraint violation figure, we use the log axis to make zero constraint violation more obvious.}
    \label{fig:compare2}
\end{figure}

In Figure \ref{fig:compare2}, we see the average performance as the line  and the shaded region is the standard deviation. We find that the proposed algorithm has the objective value that is close to the NPD-PD algorithm, which matches the result in Theorem \ref{main_theorem}. For the constraint violation, the proposed algorithm achieves zero constraint violation much faster than the NPG-PD algorithm, which again validates Theorem \ref{main_theorem}.

\end{document}